\documentclass[a4paper]{article}

\RequirePackage{amsthm,amsmath,amsfonts,amssymb}
\RequirePackage[authoryear]{natbib}
\RequirePackage{xcolor}
\RequirePackage[colorlinks,citecolor=blue,urlcolor=blue,hypertexnames=false,hyperfootnotes=false]{hyperref}
\RequirePackage{graphicx}
\usepackage{fullpage}
\usepackage[doublespacing]{setspace}
\usepackage{soul}
\usepackage[normalem]{ulem}
\usepackage[utf8]{inputenc}
\usepackage[T1]{fontenc}
\usepackage{authblk}
\usepackage{array}
\usepackage{enumerate}
\usepackage{algorithm}
\usepackage{algpseudocode}
\usepackage{multirow}
\usepackage[section]{placeins}
\usepackage{etoolbox}
\makeatletter
\let\th@thmstyleone\th@plain
\let\th@thmstyletwo\th@remark
\let\th@thmstylethree\th@definition
\makeatother

\usepackage{booktabs}
\usepackage{microtype}
\usepackage{bbm}
\usepackage{bm}
\usepackage{mathtools}
\usepackage{xparse}

\allowdisplaybreaks[4]

\let\REQUIRE\Require
\let\STATE\State
\let\FOR\For
\let\ENDFOR\EndFor
\let\IF\If
\let\ENDIF\EndIf

\def\tod{\overset{\text{d}}{\rightarrow}}

\mathtoolsset{showonlyrefs}

\theoremstyle{thmstyleone}
\newtheorem{thm}{Theorem}
\newtheorem{lem}{Lemma}
\newtheorem{cor}{Corollary}
\newtheorem{prop}{Proposition}
\theoremstyle{thmstyletwo}
\newtheorem{rem}{Remark}

\theoremstyle{thmstylethree}
\newtheorem{asmp}{Assumption}
\newtheorem{defn}{Definition}

\NewDocumentCommand{\mybar}{ O{0.8} O{0pt} m }{%  <---- Set the default values here
    \mathrlap{\hspace{#2}\overline{\scalebox{#1}[1]{\phantom{\ensuremath{#3}}}}}\ensuremath{#3}
}

 \newcommand{\ind}{\perp\!\!\!\!\perp} 
\def\1{\bm{1}}

\DeclareMathAlphabet{\mathsfit}{\encodingdefault}{\sfdefault}{m}{sl}
\SetMathAlphabet{\mathsfit}{bold}{\encodingdefault}{\sfdefault}{bx}{n}

\def\0{{\bf 0}}
\def\1{{\bf 1}}

\def\OM{{\mathcal O}}

\def\RM{{\mathcal R}}

\def\RB{{\mathbb R}}
\def\EB{{\mathbb E}}

\def\PB{{\mathbb P}}

\def\argmin{\mathop{\rm argmin}}

\newcommand{\Var}{\mathrm{Var}}

\DeclareMathOperator{\Tr}{Tr}

\AfterEndEnvironment{asmp}{\par\addvspace{0.5\baselineskip}}

\begin{document}

\title{Statistical Inference for Stochastic Gradient Descent: Beyond Finite Variance}

\author[1]{Jose Blanchet}
\author[1]{Peter Glynn}
\author[1]{Wenhao Yang}
\affil[1]{{\normalsize Department of Management Science and Engineering, Stanford University}}

\maketitle

\begin{abstract}
Stochastic gradient descent (SGD) is foundational to large-scale statistical learning and stochastic optimization. However, in some modern statistical learning problems, stochastic gradients can exhibit infinite-variance behavior. Consequently, classical inference methods for SGD that rely on a finite-variance assumption break down. We develop a model-agnostic methodology for constructing confidence regions from SGD iterates in both the finite- and infinite-variance regimes. We first show that Polyak--Ruppert averaging has an asymptotic directional scale no larger than that of the fastest-rate final iterate, analogous to its lower asymptotic variance in the finite-variance setting. Accordingly, we focus our inference methodology on the Polyak--Ruppert averaged estimator. Specifically, we establish a joint central limit theorem for this estimator and an empirical second-moment normalizer from the same iterates. This joint limit yields a self-normalized statistic in which the leading tail-dependent scaling terms cancel. We then use subsampling to estimate the relevant quantiles, avoiding explicit estimation of nuisance parameters including tail indices, slowly varying functions, or stable-law parameters. The resulting confidence regions are straightforward to implement and asymptotically valid in both the finite- and infinite-variance regimes. Empirical studies show reliable coverage in various settings, supporting the proposed method as a practical tool for uncertainty quantification in stochastic optimization.

\end{abstract}

\noindent\textbf{Keywords:} infinite-variance; stochastic gradient descent; statistical inference

\section{Introduction}
\label{sec: intro}
Many parameter estimation problems in statistics can be formulated as stochastic optimization problems. Given a convex population loss function $\ell(\cdot)$, we consider the problem
\begin{align}
    \theta^*\in\argmin_{\theta\in\RB^d}\ell(\theta).
\end{align}
A classical example is M-estimation, where the loss is an empirical average of a sample-wise criterion, $\ell(\theta):=\frac{1}{N}\sum_{i=1}^N f(\theta,X_i)$. To solve this optimization problem, \citet{robbins1951stochastic} introduced stochastic approximation (SA), where the parameter is updated iteratively according to
\begin{align}
\label{eq: sa}
    \theta_n=\theta_{n-1}-\eta_n g(\theta_{n-1},\xi_n),\qquad n\ge1,
\end{align}
where $\EB[g(\theta,\xi)]=\nabla \ell(\theta)$, $\eta_n=c\cdot n^{-\rho}$ with $\rho\in(1/2,1]$, and $\xi_n$ denotes the randomness in the stochastic gradient query. This recursion \eqref{eq: sa} is commonly referred to as stochastic gradient descent (SGD), and it is now widely used in large-scale statistical learning and stochastic optimization. A large body of work has studied its convergence properties in various settings \citep{bertsekas2000gradient, harold1997stochastic, mou2022optimal, moulines2011non, pelletier1998almost, pelletier1998weak, robbins1971convergence}. However, convergence guarantees do not quantify the statistical uncertainty of the estimator produced by the SGD trajectory. This motivates the problem of constructing asymptotically valid confidence regions for \(\theta^*\) from the iterates generated by SGD.

% In modern data-driven problems, both the scale of datasets and the complexity of statistical models have grown dramatically, making computational efficiency an essential concern. SGD is particularly well suited to this setting because each update requires only a single stochastic gradient (or a small mini-batch), which keeps the per-iteration complexity low and independent of the full sample size. Furthermore, the algorithm operates in an online manner, updating parameters sequentially without the need to store the entire dataset, thereby greatly reducing memory requirements. Beyond its computational efficiency, SGD has also demonstrated strong empirical performance. In nonconvex optimization problems, it often generalizes well in practice \citep{hardt2016train, zhou2020towards} and exhibits the ability to escape saddle points \citep{jin2017escape, daneshmand2018escaping}, contributing to its effectiveness in training complex models such as deep neural networks. Meanwhile, from a theoretical perspective, a substantial body of work \citep{bertsekas2000gradient, harold1997stochastic, mou2022optimal, moulines2011non, pelletier1998almost, pelletier1998weak, robbins1971convergence} has focused on analyzing the convergence behavior of SGD in various settings.

Formally, we aim to construct a confidence region $\text{CI}_n$ for the true parameter $\theta^*$ with nominal level $1-\delta$ such that
\begin{align}
    \lim_{n\to+\infty}\PB\left(\theta^*\in \text{CI}_n\right)=1-\delta,
\end{align}
where $n$ is the number of SGD iterations.
% Beyond convergence analysis, as SGD has become a standard optimization tool in large-scale and online learning, understanding its long-run statistical behavior is crucial for quantifying uncertainty. There is a flourishing line on establishment of statistical inference for SGD.  
% Specifically, the statisticians aim to establish an asymptotically valid confidence region $\text{CI}_n$ for true parameter $\theta^*$ with given level $1-\delta$ such that:
% \begin{align}
%     \lim_{n\to+\infty}\PB\left(\theta^*\in \text{CI}_n\right)=1-\delta,
% \end{align}
% where $n$ is the number of SGD iterations and the confidence region $\text{CI}_n$ is constructed along the SGD learning trajectory. 
A standard route to asymptotically valid confidence regions is to first establish a central limit theorem and then estimate the relevant quantiles of the limiting distribution. In early work, \citet{sacks1958asymptotic} established asymptotic normality for the final SGD iterate under suitable conditions.
% \[\resizebox{0.95\textwidth}{!}{$
%     \sqrt{\eta_n^{-1}}\left(\theta_n-\theta^*\right)\tod N\left(0,\int_0^{+\infty}\exp\left(-\left(H-\frac{\1(\rho=1) }{2c}\right)t\right)\Sigma\exp\left(-\left(H-\frac{\1(\rho=1) }{2c}\right) t\right)dt\right),$
% }\]
% where $\eta_n=c\cdot n^{-\rho}$ with $\rho\in(1/2,1]$, $H$ is the Hessian of $\ell(\theta^*)$, and $\Sigma$ is the covariance of stochastic gradient $g(\theta^*,\xi)$. 
To achieve the fastest convergence rate, the learning rate should be chosen as $\eta_n=c/n$, where $c>1/(2\sigma_{\min}(H))$, $\sigma_{\min}(H)$ denotes the smallest eigenvalue of $H$, and $H=\nabla^2\ell(\theta^*)$. Since $H$ is unknown and difficult to estimate, it is challenging to derive a reliable confidence region. Subsequently, \citet{polyak1992acceleration} established asymptotic normality for Polyak--Ruppert averaged SGD.
% \begin{align}
% \label{eq: pa_normality}
%     \sqrt{n}\left(\frac{1}{n}\sum_{i=1}^n\theta_i - \theta^*\right)\tod N\left(0, H^{-1} \Sigma H^{-1}\right),
% \end{align}
% where $\eta_n=c\cdot n^{-\rho}$ with $\rho\in(1/2,1)$. 
The advantage of Polyak--Ruppert averaging is that it achieves the fastest convergence rate, in the sense that its asymptotic distribution has a smaller variance than that of the final iterate and is independent of the choice of learning rate. Building on this result, recent work has developed several inference methodologies for SGD, including plug-in estimators \citep{chen2020statistical}, batch means \citep{chen2020statistical,zhu2021constructing}, random scaling \citep{lee2022fast}, replication \citep{zhu2024high}, and bootstrap methods \citep{zhong2023online}. However, these methods typically rely on the assumption that the stochastic gradients have finite variance. While analytically convenient, this assumption can be inconsistent with empirical findings in deep learning, where stochastic gradient noise often exhibits heavy-tailed\footnote{In this manuscript, a distribution is referred to as heavy-tailed if it does not have a finite variance, and light-tailed if it has a finite variance.} behavior \citep{simsekli2019tail, csimcsekli2019heavy}. This gap between theory and practice raises important questions about the reliability of classical inference procedures in more realistic settings.

Heavy-tailed stochastic gradients may arise from several sources, including the structure of neural networks and the choice of learning rate and batch size \citep{simsekli2019tail, csimcsekli2019heavy}. When stochastic gradients no longer have finite variance, many uncertainty quantification procedures \citep{chen2020statistical,zhu2021constructing,lee2022fast,zhu2024high,zhong2023online} fail because the classical central limit theorem of \citet{polyak1992acceleration} no longer holds. Recently, \citet{blanchet2024limit} showed that, under suitable conditions, the final iterate of SGD converges weakly to a non-Gaussian distribution:
\begin{align}
    h(\eta_n)\left(\theta_n-\theta^*\right)\tod -\int_0^{+\infty}\exp\left(-\left(H-\frac{1-\alpha^{-1}}{c}\mathbbm{1}(\rho=1)I_d\right)t\right)dL_t,
\end{align}
where $h(\eta_n):=\eta_n^{\frac{1}{\alpha}-1}b_1(\eta_n^{-1})$ is regularly varying at zero with index $\frac{1}{\alpha}-1$, and $L_t$ is a L\'evy process with index $\alpha$. This limit theorem provides a theoretical foundation for understanding infinite-variance SGD, but it does not immediately provide implementable confidence regions because of two difficulties:
\begin{enumerate}[(a)]
    \item When $\rho=1$, the final iterate attains its fastest convergence rate, but the learning-rate constant must satisfy $c>(1-\alpha^{-1})/\sigma_{\min}(H)$, so its choice depends on the unknown Hessian $H$.
    \item The normalizing function $h(\cdot)$ and the quantiles of the limiting distribution depend on the unknown tail index $\alpha$, the slowly varying function $b_1(\cdot)$, and other nuisance parameters determined by the tail behavior of the stochastic gradient $g(\theta,\xi)$.
\end{enumerate}
In certain settings where heavy-tailed noise is artificially injected into the iterates, the nuisance parameters and limiting distribution become explicitly known, making statistical inference more direct. We discuss this approach in Section~\ref{sec: q_aware} of the Supplementary Material. However, in this setting, the injected perturbations must dominate the intrinsic stochastic-gradient noise in the tails so that the asymptotic distribution is explicitly determined by the injected noise. Consequently, the resulting confidence regions may become overly conservative and fail to capture useful information about the curvature of the loss landscape. It is therefore desirable to develop inference procedures that do not rely on artificially injected randomness.

In this work, we address these challenges by developing a model-agnostic methodology for constructing confidence regions from SGD iterates. We summarize our main contributions below:
\begin{enumerate}[(a)]
    \item We first establish a central limit theorem for the Polyak--Ruppert averaged SGD estimator under infinite-variance stochastic gradients. We then show that its stable limit has no larger directional scale than that of the fastest-rate ($\rho=1$) final iterate in every direction. This parallels its smaller asymptotic variance in the finite-variance setting. Here directional scale plays the role of asymptotic variance when second moments do not exist. We further establish a joint central limit theorem for the averaged estimator and an empirical second-moment normalizer computed from the same SGD iterates. This joint result enables the construction of a self-normalized statistic in which the unknown scaling function $h(\cdot)$, the tail index $\alpha$, and other leading nuisance quantities cancel out. As a result, the statistic can be computed from the observed SGD trajectory without estimating the detailed tail behavior of the stochastic gradients.
    \item Although self-normalization removes the leading scale terms, the resulting limiting distribution still depends on unknown parameters, including the tail behavior of the stochastic gradient, the parameters of the limiting stable distribution, and the Hessian matrix. To address this issue, we introduce a subsampling procedure based on short auxiliary SGD trajectories to estimate the quantiles. The procedure applies without prior knowledge of whether the stochastic gradients have finite or infinite variance. We support the theory with numerical experiments that examine coverage and interval length across a range of tail behaviors.
\end{enumerate}
Together, these contributions bridge the gap between heavy-tailed limit theorems and statistical inference for SGD, providing a practical tool for uncertainty quantification in stochastic optimization.

\subsection{Related literature}
For weak convergence of heavy-tailed SGD, \citet{krasulina1969stochastic} established the first limit theorem in the one-dimensional case with Pareto-type noise. More recently, \citet{blanchet2024limit} extended this result to high dimensions and relaxed the assumptions on the stochastic gradients to cover a broader class of heavy-tailed distributions. In related work, \citet{goodsell1976almost} and \citet{li1994almost} investigated conditions for almost sure convergence, while \citet{wang2021convergence} analyzed the weak convergence of Polyak--Ruppert averaged SGD under Pareto-type noise, without giving a detailed characterization of the limiting distribution or developing an inference methodology.

Beyond limit theorems for SGD, heavy-tailed phenomena in stochastic optimization have also been studied extensively. Both theoretical analyses and empirical evidence show that hyperparameters such as the learning rate, batch size, and loss structure can induce heavy-tailed behavior \citep{gurbuzbalaban2021heavy, hodgkinson2021multiplicative, schertzer2024stochastic, jiao2024emergence, damek2024analysing}. On the algorithmic side, \citet{cutkosky2021high}, \citet{liu2023stochastic}, and \citet{wang2021convergence} investigated non-asymptotic convergence rates under various noise assumptions. Moreover, several works demonstrate that heavy-tailed behavior is positively correlated with generalization properties across different learning settings \citep{mahoney2019traditional, simsekli2020hausdorff, barsbey2021heavy}. In particular, heavy-tailed SGD dynamics have been shown to favor wider local minima over sharper ones \citep{csimcsekli2019heavy,simsekli2019tail}, while \citet{wang2021eliminating} proposed truncated heavy-tailed SGD to mitigate convergence to sharp minima.

In the broader literature on heavy-tailed statistics, \citet{logan1973limit} introduced self-normalization techniques that cancel dominant nuisance parameters in heavy-tailed mean estimation problems. Building on this idea, \citet{romano1999subsampling} proposed a subsampling approach to estimate quantiles of the limiting distribution of self-normalized statistics in i.i.d. heavy-tailed mean estimation. More recently, \citet{bai2016unified} extended this framework to accommodate long-range dependent time series.

\subsection{Notation}
Let $\bar{\mathbb{R}} := \mathbb{R}\cup\{\pm\infty\}$ and $\bar{\mathbb{R}}^d := (\bar{\mathbb{R}})^d$.
We write $\overset{d}{\to}$ for convergence in distribution, $\overset{p}{\to}$ for convergence in probability, $\overset{\mathrm{a.s.}}{\to}$ for almost sure convergence, $\overset{v}{\to}$ for vague convergence, and $\Rightarrow$ for weak convergence in function spaces.
We write $X_n=o_p(1)$ if $X_n\overset{p}{\to}0$, and $X_n=O_p(1)$ if $X_n$ is bounded in probability.
The $(d-1)$-dimensional unit sphere is denoted by
$\mathbb{S}^{d-1}:=\{x\in\mathbb{R}^d:\|x\|=1\}$.
We denote by $D_{J_1}([0,1],\mathbb{R}^d)$ the Skorokhod space endowed with the $J_1$ topology.
$\mathbb{S}_{++}^d$ denotes the set of all $d\times d$ symmetric positive definite matrices, $I_d$ denotes the $d\times d$ identity matrix, and $[d]:=\{1,\ldots,d\}$.
We write $a\lesssim b$ if there exists $C>0$ such that $a\le Cb$.

\section{Problem setup}
\label{sec: prel}
In this section, we introduce the stochastic gradient descent setup and the assumptions used throughout the paper. We then summarize the relevant limit theory for heavy-tailed SGD and explain why these results alone do not immediately yield a practical inference procedure. Foundational results from heavy-tailed statistics are deferred to Section~\ref{apd: gclt} of the Supplementary Material.

\subsection{Background and assumptions}
Given an optimization problem
\begin{align}
    \min_{\theta\in\RB^d}\ell(\theta),
\end{align}
where $\ell(\theta)$ is a twice-differentiable, strongly convex function, stochastic gradient descent iteratively approximates the minimizer $\theta^*\in\argmin_{\theta}\ell(\theta)$ via
\begin{align}
\label{eq: sgd}
    \theta_n=\theta_{n-1}- \eta_n g(\theta_{n-1}, \xi_n),\qquad n\ge1,
\end{align}
where $\EB_{\xi\sim\PB}[g(\theta,\xi)]=\nabla \ell(\theta)$ and $\{\xi_i\}_{i\ge1}$ are i.i.d. random variables. Motivated by empirical observations that stochastic gradients may exhibit infinite variance \citep{simsekli2019tail}, we impose the following multivariate regular variation condition adapted from \citet{blanchet2024limit}.
\begin{asmp}
    \label{asmp: tail_norm}
    For each $\theta\in\mathbb{R}^d$, there exist $\alpha\in(1,2)$, a
    differentiable slowly varying function
    $b_0(\cdot,\theta):\mathbb{R}_+\to\mathbb{R}_+$, and a bounded Lipschitz
    function $C:\mathbb{R}^d\to\mathbb{R}_+$ satisfying
    \begin{align}
        \PB\left(\|g(\theta,\xi)\|> r\right)
        &=r^{-\alpha}b_0(r,\theta),
        \label{eq: g_norm}
    \end{align}
    where $b_0$ additionally satisfies
    \begin{align}
        \lim_{r\to+\infty}\frac{b_0(r,\theta)}{b_0(r,\theta^*)}
        &=C(\theta),
        \qquad
        \lim_{r\to+\infty}\frac{r\,\partial_r b_0(r,\theta)}
        {b_0(r,\theta)}=0.
        \notag
    \end{align}
    Let $b_1:\mathbb{R}_+\to\mathbb{R}_+$ be the slowly varying normalization chosen so that
    \[
        \PB\left(\|g(\theta^*,\xi)\|>x^{1/\alpha}b_1(x)^{-1}\right)=x^{-1}.
    \]
    Then, uniformly for $\theta$ in compact sets,
    \begin{align}
        x\PB\left(x^{-1/\alpha}b_1(x)g(\theta,\xi)\in\cdot\right)
        \overset{v}{\to}\nu(\theta,\cdot),\qquad x\to+\infty,
        \label{eq: g_radon}
    \end{align}
    where $\nu(\theta,\cdot)$ is a nonzero Radon measure on
    $\bar{\mathbb{R}}^d\setminus\{0\}$ that is homogeneous of order
    $-\alpha$.
\end{asmp}
We also impose regularity assumptions on the loss $\ell(\theta)$ and the stochastic gradient $g(\theta,\xi)$, following \citet{wang2021convergence} and \citet{blanchet2024limit}.
\begin{asmp}
    \label{asmp: ppd}
    The set $\{\nabla^2\ell(\theta):\theta\in\RB^d\}$ is bounded and uniformly $p$-positive definite for any $p\in(1,\alpha)$. Equivalently, for any $\theta\in\RB^d$ and $\|u\|_p=1$, $\nabla^2\ell(\theta)$ is bounded and satisfies $u^\top\nabla^2\ell(\theta)\left(\text{sgn}(u_1)|u_1|^{p-1},\ldots,\text{sgn}(u_d)|u_d|^{p-1}\right)^\top>0$ for any $p\in(1,\alpha)$.
\end{asmp}
\begin{asmp}
    \label{asmp: lip_error}
    There exists $C_{\rm Lip}>0$ such that for any $\theta_1,\theta_2$,
    \begin{align}
        \|g(\theta_1, \xi)-g(\theta_2,\xi)\|\le C_{\rm Lip}\|\theta_1-\theta_2\|.
    \end{align}
\end{asmp}
\begin{asmp}
\label{asmp: talyor_error}
    There exist $q\in(1,\alpha)$ and $K>0$ such that for any $\theta$,
    \begin{align}
        \|\nabla\ell(\theta)-\nabla^2\ell(\theta^*)(\theta-\theta^*)\|\le K\|\theta-\theta^*\|^q.
    \end{align}
    We denote $H:=\nabla^2\ell(\theta^*)$.
\end{asmp}
Assumption~\ref{asmp: ppd} strengthens the standard strong convexity condition and is commonly used in the analysis of heavy-tailed stochastic optimization; see \citet{wang2021convergence}. Assumption~\ref{asmp: lip_error} imposes a smoothness requirement on the stochastic gradients, while Assumption~\ref{asmp: talyor_error} bounds the approximation error of the nonlinear loss function $\ell(\cdot)$. Similar assumptions are also adopted by \citet{polyak1992acceleration} and \citet{wang2021convergence}.

\subsection{Limit theory for SGD iterates}
In prior work by \citet{blanchet2024limit}, a central limit theorem for the final iterate of SGD was established, as summarized below.
\begin{thm}[\citet{blanchet2024limit}]
\label{thm: final_limit}
    Suppose the learning rate satisfies $\eta_n=c\cdot n^{-\rho}$ with $\rho\in(\alpha^{-1},1]$. When $\rho=1$, also suppose $c>(1-\alpha^{-1})/\sigma_{\min}(H)$. Under Assumptions~\ref{asmp: tail_norm}, \ref{asmp: ppd}, \ref{asmp: lip_error}, and \ref{asmp: talyor_error}, we have
    \begin{align}
        n^{\rho(1-\frac{1}{\alpha})}b_1(n^{\rho})\left(\theta_n-\theta^*\right)\tod Z_{\text{final},\rho}\overset{d}{=}-c^{1-\frac{1}{\alpha}}\int_0^{\infty}\exp\left(-H_{\rho}t\right)dL_t,
    \end{align}
    where $L_t$ is a L\'evy process with characteristics $\left(0, \nu(\theta^*,\cdot), -\int_{\|x\|>1}x\nu(\theta^*,dx)\right)$, and $H_{\rho}=H-\frac{1-\alpha^{-1}}{c}\mathbbm{1}(\rho=1)I_d$.
\end{thm}
In practice, one naturally prefers algorithms that achieve the fastest possible convergence rate. Theorem~\ref{thm: final_limit} shows that, for SGD, the optimal rate is attained when the learning rate is chosen as $\eta_n = c \cdot n^{-1}$. However, this choice requires the constant $c$ to satisfy
\begin{align}
    c>\frac{1-\alpha^{-1}}{\sigma_{\min}(H)},
\end{align}
which is difficult to verify or control in practice. To address this issue, we draw on the averaging technique of \citet{polyak1992acceleration}, which removes such parameter dependence and yields convergence guarantees without the delicate tuning of $c$. The proof is deferred to Section~\ref{apd: pa_limit} of the Supplementary Material.
\begin{thm}[Polyak--Ruppert averaging]
\label{thm: pa_limit}
    Suppose the learning rate satisfies $\eta_n=c\cdot n^{-\rho}$ with $\rho\in(q^{-1},1)$, where $q\in(1,\alpha)$ is the exponent in Assumption~\ref{asmp: talyor_error}, and let $\bar{\theta}_n=n^{-1}\sum_{k=1}^n\theta_k$. Under Assumptions~\ref{asmp: tail_norm}, \ref{asmp: ppd}, \ref{asmp: lip_error}, and \ref{asmp: talyor_error}, we have
    \begin{align}
        n^{1-\frac{1}{\alpha}}b_1(n)\left(\bar{\theta}_n-\theta^*\right)\tod Z_{\text{Polyak}},
    \end{align}
    where $Z_{\text{Polyak}}\overset{d}{=}-H^{-1}L_1$ and $L_t$ is defined in Theorem~\ref{thm: final_limit}.
\end{thm}
The restriction $\rho>q^{-1}$ is required to make the Taylor-remainder term negligible on the stable-limit scale. If Assumption~\ref{asmp: talyor_error} is strengthened to hold for every $q\in(1,\alpha)$, then any $\rho>\alpha^{-1}$ is covered by choosing $q\in(\rho^{-1},\alpha)$.

In the finite-variance setting, the Polyak--Ruppert averaged estimator for SGD is known to achieve the Cram\'er--Rao lower bound under mild regularity conditions \citep{polyak1992acceleration}, implying that it is the most efficient estimator attainable. In contrast, in the infinite-variance regime, such minimax efficiency criteria are no longer applicable because the second moment is unbounded. We instead establish that the Polyak--Ruppert averaged estimator has asymptotic directional scale no larger than that of the final-iterate estimator. The scale function characterizes the directional dispersion of a stable limiting distribution and therefore serves as a variance analogue in the infinite-variance regime. We defer the detailed definition of distributional scale to Section~\ref{apd: gclt} of the Supplementary Material.
\begin{thm}
\label{thm: scale_diff}
    Suppose the learning-rate constant in Theorem~\ref{thm: final_limit} with $\rho=1$ satisfies $c>(1-\alpha^{-1})/\sigma_{\min}(H)$, so that $Z_{\text{final},1}$ is well defined.
    The scales of the limit distributions $Z_{
    \text{final},1
    }$ in Theorem~\ref{thm: final_limit} and $Z_{\text{Polyak}}$ in Theorem~\ref{thm: pa_limit} satisfy, for any $\|u\|=1$:
    \begin{align}
        \label{eq: comp}
        \sigma_{\alpha}\left(Z_{\text{final},1};u\right)\ge\sigma_{\alpha}\left(Z_{\text{Polyak}};u\right).
    \end{align}
\end{thm}

In the finite-variance setting, both $Z_{\text{final},1}$ and $Z_{\text{Polyak}}$ are Gaussian, and the directional-scale comparison in Theorem~\ref{thm: scale_diff} becomes the variance comparison
\begin{align}
    \Var\left(u^\top Z_{\text{final},1}\right)
    \ge
    \Var\left(u^\top Z_{\text{Polyak}}\right),
    \qquad \|u\|=1.
\end{align}
This recovers the classical result of \citet{polyak1992acceleration} that the Polyak--Ruppert averaged estimator has no larger asymptotic variance than the fastest-rate final iterate.

\subsection{The inference challenge}
The limit theorems above identify the asymptotic behavior of SGD iterates in the heavy-tailed regime, but they do not provide an implementable confidence region. To turn Theorem~\ref{thm: pa_limit} into an inference procedure, one would need to know the normalization rate $n^{1-\frac{1}{\alpha}}b_1(n)$ and the relevant quantiles of the stable limiting distribution. Both quantities depend on the tail behavior of the stochastic gradients and are generally unavailable from the observed SGD trajectory.

One situation in which these quantities may become accessible is when the heavy-tailed noise is artificially injected. To distinguish this case from the usual data-driven randomness, we decompose the stochastic gradient into two sources: \textit{intrinsic randomness} and \textit{artificially injected randomness}. For simplicity, we represent the stochastic gradient as
\begin{align}
\label{eq: decomp_sg}
    g(\theta,\xi):=g(\theta,\xi_{\text{int}})+\xi_{\text{art}},
\end{align}
where $\xi_{\text{int}}$ denotes the intrinsic randomness and $\xi_{\text{art}}$ denotes the artificially injected randomness. For example, in perturbed gradient descent (PGD), one may inject random noise $\xi_{\text{art}}$ with heavier tails than $g(\theta,\xi_{\text{int}})$ into each gradient descent iterate. In this case, the parameters $\alpha$, $b_1(\cdot)$, and the distribution of $L_1$ in Theorem~\ref{thm: pa_limit} are all determined by the injected noise. These quantities are therefore known. Consequently, Theorem~\ref{thm: pa_limit} can be used directly to establish confidence regions for $\theta^*$, provided that a consistent estimator of $H$ is available. We discuss this approach in Section~\ref{sec: q_aware} of the Supplementary Material.

However, artificially injecting noise is not always desirable. It requires prior knowledge of the intrinsic stochastic gradient noise in order to ensure that the injected component dominates the tail behavior. Consequently, the resulting confidence regions may become overly conservative and fail to capture useful information about the curvature of the loss landscape. We therefore seek an inference procedure that does not rely on artificial perturbations. In this setting, the nuisance parameters in Theorem~\ref{thm: pa_limit} are determined by the intrinsic stochastic gradients, leading to two major obstacles:
\begin{enumerate}[(a)]
    \item The tail index $\alpha$ and the slowly varying function $b_1(\cdot)$ are unknown.
    \item The quantiles of $Z_{\text{Polyak}}$ are unknown.
\end{enumerate}
In practice, it is often infeasible to verify whether the underlying data-generating mechanism follows an $\alpha$-regularly varying distribution, and accurate estimation of $b_1(\cdot)$ is especially difficult. These obstacles motivate the self-normalized inference procedure developed in Section~\ref{sec: limit}.

\section{Self-normalized subsampling inference}
\label{sec: limit}
In this section, we develop a self-normalized inference procedure that addresses the difficulties identified in Section~\ref{sec: prel}. The key idea is to normalize the Polyak--Ruppert averaged estimator by an empirical second-moment estimator computed from the same SGD trajectory. This normalization cancels the leading tail-dependent scaling factors, while subsampling provides a data-driven way to estimate the remaining critical values.

\subsection{Self-normalized statistic}
\label{sec: d_drive}
Define the empirical second-moment normalizer by
\begin{align}
    \Sigma_n=\frac{1}{n}\sum_{k=1}^ng(\theta_{k-1},\xi_k)g(\theta_{k-1},\xi_k)^\top.
\end{align}
We define the self-normalized statistic
\begin{align}
\label{eq: t_stat}
    T_n^{\star}(\theta;\varphi):=\frac{\sqrt{n}\cdot \varphi\left(\bar{\theta}_n-\theta\right)}{\sqrt{\Tr(\Sigma_n)}},
\end{align}
where $\varphi:\RB^d\to\RB$ is a continuous homogeneous function of degree 1. The role of $\Sigma_n$ is to provide a data-dependent normalization that cancels the unknown tail scaling in the heavy-tailed regime. The following theorem gives the formal justification.

\begin{thm}
\label{thm: joint}
    Suppose the learning rate satisfies $\eta_n=c n^{-\rho}$ with $\rho\in(q^{-1},1)$, where $q$ is the exponent in Assumption~\ref{asmp: talyor_error}. Under Assumptions~\ref{asmp: tail_norm}, \ref{asmp: ppd}, \ref{asmp: lip_error}, and \ref{asmp: talyor_error}, the joint weak convergence holds:
    \begin{align}
        \left(n^{1-\frac{1}{\alpha}}b_1(n)\left(\bar{\theta}_n-\theta^*\right), n^{1-\frac{2}{\alpha}}b_1(n)^2\Sigma_n\right)&\tod \left(-H^{-1}L_1, W\right),
    \end{align}
    where $W$ is infinitely divisible. For any continuous homogeneous function $\varphi:\RB^d\to\RB$ of degree 1, the continuous mapping theorem gives
    \begin{align}
        \left|T_n^{\star}(\theta^*;\varphi)\right|\tod\frac{\left|\varphi\left(-H^{-1}L_1\right)\right|}{\sqrt{\Tr(W)}}:=Z^\star(\varphi).
    \end{align}
\end{thm}

In Theorem~\ref{thm: joint}, the distribution of $W$ need not be full rank. For instance, $W$ fails to be full rank when some coordinates of $g(\theta^*,\xi)$ have tail indices strictly larger than $\alpha$. The benefit of $T_n^\star(\theta^*;\varphi)$ is that it cancels out the dependence on the tail index $\alpha$ and the slowly varying function $b_1(\cdot)$. Consequently, if the $(1-\delta)$ quantile of $Z^\star(\varphi)$ were known, a valid confidence region would be
\begin{align}
\label{eq: cr_infty}
    \RM_n(\delta;\varphi) = \left\{\left.\theta\in\RB^d\right|\left|T_n^{\star}(\theta;\varphi)\right|\le q(\delta;\varphi)\right\},
\end{align}
where $q(\delta;\varphi)=\inf\left\{x\left|\PB\left(Z^\star(\varphi)>x\right)\le \delta\right.\right\}$. The remaining task is to estimate $q(\delta;\varphi)$ without knowledge of the limiting distribution.
\begin{rem}
    In equation~\eqref{eq: t_stat}, there are many choices for the homogeneous function $\varphi(\cdot)$. Typical choices of $\varphi(\cdot)$ include linear projections and norms such as $\|\cdot\|_{\infty}$ and $\|\cdot\|_p$. Similarly, the trace normalizer may also be replaced by other matrix functionals, such as the Frobenius norm or the largest eigenvalue, provided the corresponding limiting statistic is well defined.
\end{rem}

\subsection{Subsampling calibration}
Motivated by the subsampling idea of \cite{romano1999subsampling}, we estimate $q(\delta;\varphi)$ using short auxiliary SGD trajectories. Let $t_n=\lfloor n^r\rfloor$ for some $r\in(0,1)$ and let $B_n=\lfloor n/t_n\rfloor$. For each sub-procedure $b\in\{1,\ldots,B_n\}$, let $\bar{\theta}^{(b)}_{t_n}$ and $\Sigma^{(b)}_{t_n}$ denote the averaged estimator and empirical second-moment estimator computed from a length-$t_n$ SGD trajectory. We estimate the distribution of the self-normalized statistic by
\begin{align}
\label{eq: emp_dist}
    \widehat{F}_n(x;\varphi)=\frac{1}{B_n}\sum_{b=1}^{B_n}\mathbbm{1}\left(\left|T_{n,\varphi}^{\star,(b)}\right|:=\frac{\sqrt{t_n}\cdot \left|\varphi\left(\bar{\theta}^{(b)}_{t_n}-\bar{\theta}_{n}\right)\right|}{\sqrt{\Tr\left( \Sigma^{(b)}_{t_n}\right)}}\le x\right).
\end{align}
Let $\widehat{q}(\delta;\varphi)=\inf\left\{x:\widehat{F}_n(x;\varphi)\ge 1-\delta\right\}$. The resulting data-driven confidence region is
\begin{align}
\label{eq: cr_hat}
    \widehat{\RM}_n(\delta;\varphi)=
    \left\{\theta\in\RB^d:
    \frac{\sqrt{n}\cdot \left|\varphi\left(\bar{\theta}_n-\theta\right)\right|}{\sqrt{\Tr(\Sigma_n)}}\le \widehat{q}(\delta;\varphi)
    \right\}.
\end{align}
\begin{algorithm}[t]
    \caption{Data-driven confidence region for SGD}
    \label{alg: general}
    \begin{algorithmic}[1]
    \REQUIRE Fixed number of iterations $n$, confidence level $\delta$, block size $t_n=\lfloor n^{r}\rfloor$ ($r\in(0,1)$), number of complete blocks $B_n=\lfloor n/t_n\rfloor$, test function $\varphi(\cdot)$.
    \STATE Initialize $b=1$, $\theta_0=\theta_0^{(b)}=\bar{\theta}_0=\bar{\theta}_0^{(b)}=\theta_0$, and $\Sigma_0=\Sigma_0^{(b)}=0$.
    \FOR{$k=1,2,\ldots,n$}
        \STATE Query $g(\theta_{k-1},\xi_k)$ and update the main trajectory:
        \begin{align}
            &\theta_{k} = \theta_{k-1}-\eta_k g(\theta_{k-1},\xi_k),\\
            &\bar{\theta}_{k}=\frac{(k-1)\bar{\theta}_{k-1}+\theta_k}{k},\\
            &\Sigma_{k} = \frac{(k-1)\Sigma_{k-1}+g(\theta_{k-1},\xi_k)g(\theta_{k-1},\xi_k)^\top}{k}.
        \end{align}
        \IF{$k\le B_nt_n$}
            \STATE Set $k_b=k-(b-1)t_n$ and query $g(\theta_{k_b-1}^{(b)},\xi_{k_b}^{\prime,(b)})$.
            \STATE Update the current auxiliary trajectory:
            \begin{align}
                &\theta_{k_b}^{(b)} = \theta_{k_b-1}^{(b)}-\eta_{k_b}g(\theta_{k_b-1}^{(b)},\xi_{k_b}^{\prime,(b)}),\\
                &\bar{\theta}^{(b)}_{k_b}=\frac{(k_b-1)\bar{\theta}^{(b)}_{k_b-1}+\theta_{k_b}^{(b)}}{k_b},\\
                &\Sigma_{k_b}^{(b)} = \frac{(k_b-1)\Sigma^{(b)}_{k_b-1}+g(\theta^{(b)}_{k_b-1},\xi_{k_b}^{\prime,(b)})g(\theta^{(b)}_{k_b-1},\xi_{k_b}^{\prime,(b)})^\top}{k_b}.
            \end{align}
        \ENDIF
        \IF{$k\le B_nt_n$ and $k_b=t_n$ and $b<B_n$}
            \STATE Set $b=b+1$, $\theta_0^{(b)}=\bar{\theta}_0^{(b)}=\theta_0$, and $\Sigma_{0}^{(b)}=0$.
        \ENDIF
    \ENDFOR
    \FOR{$i=1,2,\ldots,B_n$}
        \STATE Calculate $T_{n,\varphi}^{\star,(i)}$ in equation~\eqref{eq: emp_dist}.
    \ENDFOR
    \STATE Calculate $\widehat{F}_n$ in equation~\eqref{eq: emp_dist} and the quantile $\widehat{q}(\delta;\varphi)$.
    \STATE \textbf{Output:} Confidence region $\widehat{\RM}_n(\delta;\varphi)$ in equation~\eqref{eq: cr_hat}.
    \end{algorithmic}
\end{algorithm}
A detailed procedure is described in Algorithm~\ref{alg: general}. Auxiliary trajectories are run only for the first $B_nt_n$ iterations, so the remaining $n-B_nt_n$ iterations update the main trajectory without creating an incomplete block. The main and auxiliary gradients may use different randomness. Specifically, $\xi_{k_b}^{\prime,(b)}$ may be drawn independently of or set equal to the corresponding $\xi_k$. In the latter case, the statistics $\left(\bar{\theta}^{(b)}_{t_n},\Sigma^{(b)}_{t_n}\right)$ remain independent across sub-procedures because the blocks are disjoint. Algorithm~\ref{alg: general} records the full matrix $\Sigma_n$ for notational clarity. In practice, when the confidence region is based only on $\Tr(\Sigma_n)$, it is enough to update the trace directly.
\begin{thm}
\label{thm: sub_consistent}
    Suppose $t_n/n\to0$, $B_n\to\infty$, $B_n/n\to0$, and the continuous homogeneous functional $|\varphi(\cdot)|$ is even and subadditive, meaning that $|\varphi(x+y)|\le |\varphi(x)|+|\varphi(y)|$ for all $x,y\in\mathbb{R}^d$. Under the conditions of Theorem~\ref{thm: joint}, for $\widehat{F}_n(x;\varphi)$ in \eqref{eq: emp_dist}, if $x$ is a continuity point of the distribution function $F_\varphi$ of $Z^\star(\varphi)$, then
    \begin{align}
        \widehat{F}_n(x;\varphi)\overset{p}{\to}\PB\left(Z^\star(\varphi)\le x\right).
    \end{align}
    Let $q(\delta;\varphi)=\inf\{x:F_\varphi(x)\ge1-\delta\}$. If $F_\varphi$ is continuous and strictly increasing at $q(\delta;\varphi)$, then the output confidence region in Algorithm~\ref{alg: general} has asymptotic $1-\delta$ coverage:
    \begin{align}
    \PB\left(\theta^*\in\widehat{\RM}_n(\delta;\varphi)\right)\to1-\delta.
    \end{align}
\end{thm}
% If $F_\varphi$ is discontinuous at $q(\delta;\varphi)$, merely switching between the left and right generalized inverses does not determine from which side the estimated critical value approaches the atom. A simple conservative modification is to expand the critical value. Suppose $F_\varphi$ is strictly increasing at the quantile. For any fixed $a>0$ such that $q(\delta;\varphi)+a$ is a continuity point of $F_\varphi$, define
% \begin{align}
%     \widehat q_a(\delta;\varphi)&=\widehat q(\delta;\varphi)+a,\\
%     \widehat{\RM}_{n,a}(\delta;\varphi)&=
%     \left\{\theta:\left|T_n^\star(\theta;\varphi)\right|\le\widehat q_a(\delta;\varphi)\right\}.
% \end{align}
% Then
% \begin{align}
%     \PB\left(\theta^*\in\widehat{\RM}_{n,a}(\delta;\varphi)\right)
%     \longrightarrow F_\varphi(q(\delta;\varphi)+a)\ge1-\delta.
% \end{align}
% Thus the buffered region is asymptotically conservative. One may subsequently let $a\downarrow0$ through continuity points; a single sequence $a=a_n\downarrow0$ additionally requires a rate controlling both the statistic and critical-value approximations.
\begin{rem}
    It is worth noting that the joint convergence result in Theorem~\ref{thm: joint} continues to hold even when the stochastic gradient noise has finite variance; see \citet{polyak1992acceleration}. In that case, the scaling rate changes from $n^{1-\frac{1}{\alpha}}$ to the classical $\sqrt{n}$ rate, and the limiting distribution becomes Gaussian rather than stable. Nevertheless, the self-normalized statistic $T_n^\star(\theta;\varphi)$ remains well defined, and the inference procedure described in Algorithm~\ref{alg: general} can be applied without modification. Importantly, the algorithm does not require prior knowledge of whether the stochastic gradient noise has finite or infinite variance and therefore remains applicable across both regimes.
\end{rem}

\subsection{Discussion of ellipsoidal confidence regions}
Several geometries can be considered when constructing confidence regions for multidimensional parameters. Ellipsoidal confidence regions are a popular choice in classical statistical inference because they naturally account for correlations among coordinates through the covariance matrix. This idea underlies the well-known Hotelling $t^2$ statistic and has also been adopted in recent work on statistical inference for stochastic gradient descent; see \citet{lee2022fast}. In the infinite-variance regime considered in this paper, ellipsoidal confidence regions remain possible in principle. In particular, Corollary~\ref{cor: linear} shows that a self-normalized quadratic statistic can be constructed with a nondegenerate limiting distribution, provided that the limiting scale matrix $W$ is invertible almost surely.
\begin{cor}
\label{cor: linear}
    If $W$ in Theorem~\ref{thm: joint} is invertible almost surely, the following limit theorem holds:
    \begin{align}
        \label{eq: self-norm}
        T_n^\dagger:=n\left(\bar{\theta}_n-\theta^*\right)^\top\Sigma_n^{-1}\left(\bar{\theta}_n-\theta^*\right)\tod L_1^\top H^{-1}W^{-1}H^{-1}L_1.
\end{align}    
\end{cor}
For $T_n^\dagger$ in Corollary~\ref{cor: linear}, the nuisance parameters $\alpha$ and $b_1(\cdot)$ are also canceled out on the left via self-normalization. However, the condition that $W$ is invertible almost surely can only be satisfied in limited cases. In particular, one sufficient condition is that each coordinate of $g(\theta,\xi)$ has the same regularly varying tail index and the limit distribution $L_1$ does not lie in any proper subspace of $\RB^d$; see \citet{meerschaert1999sample}. If some coordinates of $g(\theta,\xi)$ have lighter tails, the statistic $T_n^\dagger$ will be invalid. To illustrate this issue, we consider the following simple quadratic optimization problem with dimension $d=2$:
\begin{align}
    \min_{\theta}\frac{1}{2}\theta^\top H\theta,
    \qquad
    H:=\left(\begin{smallmatrix}
        1&1\vspace{4pt}\\
        1&2
    \end{smallmatrix}\right)^{-1}.
\end{align}
The stochastic gradient satisfies $g(\theta,\xi)=H\theta+\xi$, where $\xi=(\xi_{1},\xi_{2})^\top$ and $\xi_{1}\ind\xi_{2}$. Here $\xi_{1}$ is an $\alpha_1$-regularly varying random variable and $\xi_{2}$ is an $\alpha_2$-regularly varying random variable with $1<\alpha_1<\alpha_2<2$. In this case, the joint limit takes the following form:
\begin{align}
    \left(n^{1-\frac{1}{\alpha_1}}\left(\bar{\theta}_n-\theta^*\right),
    n^{1-\frac{2}{\alpha_1}}\Sigma_n \right) \tod
    \left(-\left(\begin{smallmatrix} L_1\vspace{4pt}\\ L_1 \end{smallmatrix}\right),
    \left(\begin{smallmatrix} W_1&0\vspace{4pt}\\ 0&0 \end{smallmatrix}\right)\right),
\end{align}
where the second-moment scaling term is singular and $T_n^\dagger$ cannot be used to construct a confidence region for $\theta^*$. In contrast, the confidence region based on the trace of the second-moment matrix does not require inversion and remains well defined even when different coordinates exhibit heterogeneous tail indices.
% Moreover, the computation of  is significantly simpler, requiring only $\OM(d)$ operations rather than matrix inversion with complexity $\OM(d^3)$. These properties make the hypercubic confidence region particularly suitable for inference in heavy-tailed stochastic optimization problems. We summarize the comparison between the two types of confidence regions in Table~\ref{tab: result}.
% \begin{table}[h!]
%     \centering
%     \scalebox{1.0}{
%     \begin{tabular}{c|c|c}
%     \toprule
%      & Hypercubic & Ellipsoid \\
%     \hline
%     Computation Complexity & $\OM(d)$ & $\OM(d^3)$\\
%     Robust to Heterogeneous Noise & \ding{51}  & \ding{55}\\
%     \bottomrule
%     \end{tabular}
%     }
%     \caption{Benefits of hypercubic confidence regions in heavy-tailed cases.}
%     \label{tab: result}
% \end{table}

\section{Simulation experiments}
\label{sec: exp}
In this section, we evaluate the empirical performance of the proposed inference procedure through simulation studies. Our experiments examine the finite-sample behavior of the confidence regions constructed by Algorithm~\ref{alg: general} under different stochastic gradient noise regimes, with particular emphasis on empirical coverage probabilities and confidence interval lengths. We consider two commonly used statistical models, linear regression and logistic regression, as discussed by \citet{blanchet2024limit}. For each setting, we run stochastic gradient descent with Polyak--Ruppert averaging and construct confidence regions using the proposed method. The nominal coverage level is 95\% throughout. We also evaluate our methodology in a challenging setting where the tail index varies across gradient queries. The results are reported in Section~\ref{apd: robust} of the Supplementary Material.
\subsection{Linear regression}
We consider the linear regression model with heavy-tailed noise studied by \citet{blanchet2024limit}:
\begin{align}
    y_i = x_i^\top\theta^*+\varepsilon_i,
\end{align}
where $\{(x_i,\varepsilon_i)\}_{i=1}^n$ are i.i.d., $x_i\sim N(0,\Sigma)$, and $\varepsilon_i$ follows a symmetric Pareto distribution with mean zero and tail index $\alpha$. We further assume that $x_i\ind\varepsilon_i$. Given a sample pair $(x_i,y_i)$, the stochastic gradient is
\begin{align}
    g(\theta,x_i, y_i) = x_i(x_i^\top\theta-y_i).
\end{align}
For the choice of $\Sigma$, we consider two cases:
\begin{enumerate}[(a)]
    \item Identity: $\Sigma = I$.
    \item Toeplitz: $\Sigma_{ij} = q^{|i-j|}$ with $q=0.3$.
\end{enumerate}
For this model, we set the sample size to $n=10^6$ and consider dimensions $d\in\{5,20\}$. The true parameter $\theta^*$ is generated from the standard Gaussian distribution $N(0,I_d)$. We consider three subsample sizes: $n^{0.6}$, $n^{0.7}$, and $n^{0.8}$. The nominal coverage level is 95\%, and we conduct 500 independent runs to estimate the empirical coverage rate. We evaluate Algorithm~\ref{alg: general} using the average coverage rate and average confidence interval length. Table~\ref{table:linear_CI} reports the linear regression results. Figure~\ref{fig: linear_regression} displays the corresponding results for the identity covariance setting $\Sigma=I$ with $\alpha=1.5$ and $d=5$. We summarize the results below.

\begin{figure}[t]
    \centering
        \caption{Linear regression with identity covariance $\Sigma=I$: average coverage rate and confidence-interval length for $\alpha=1.5$ and $d=5$ under different subsample sizes.}
    \includegraphics[width=1.0\linewidth]{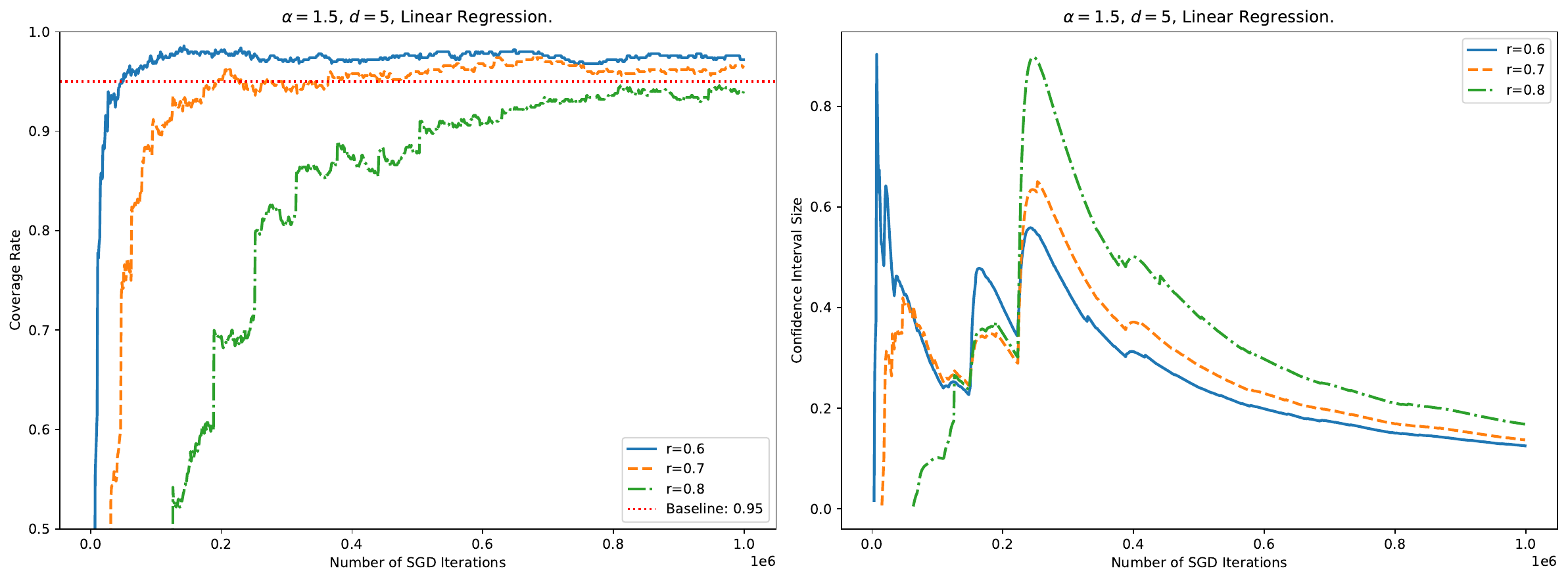}
    \label{fig: linear_regression}
\end{figure}

\begin{table}[t]
\caption{Linear regression: average coverage rates and confidence interval lengths at the nominal $95\%$ coverage level. Standard errors $\left(\frac{\text{std}_{500}}{\sqrt{500}}\right)$ are reported in parentheses.}
\resizebox{\textwidth}{!}{\begin{tabular}{lccccccc}
\toprule
 & $d$ & \multicolumn{3}{c}{Coverage Rate (\%)}  &  \multicolumn{3}{c}{Average Length ($\times 10^{-2}$)} \\
  &   & $n^{0.6}$&$n^{0.7}$&$n^{0.8}$ & $n^{0.6}$&$n^{0.7}$&$n^{0.8}$  \\
\midrule
Identity $\Sigma$\\
\multirow{2}{*}{$\alpha=1.5$}&5&97.0(7.6E-3)&96.8(7.9E-3)&94.4(1.0E-2)&11.33(1.7E-2)&12.57(2.5E-2)&15.47(4.0E-2)\\
 &20&99.8(2.0E-3)&98.2(5.9E-3)&93.8(1.1E-2)&94.26(8.0E-1)&46.84(3.1E-1)&17.00(2.7E-2)\\
\addlinespace
 \multirow{2}{*}{$\alpha=1.8$}&5&95.8(9.0E-3)&95.6(9.2E-3)&91.0(1.3E-2)&1.91(7.1E-4)&1.91(8.8E-4)&1.92(1.3E-3)\\
 &20&99.6(2.8E-3)&97.8(6.6E-3)&91.8(1.2E-2)&3.48(7.7E-3)&2.98(4.5E-3)&2.35(1.0E-3)\\
\addlinespace
\multirow{2}{*}{Gaussian}&5&94.6(1.0E-2)&94.4(1.0E-2)&86.4(1.5E-2)&0.26(4.5E-6)&0.25(8.1E-6)&0.23(1.2E-5)\\
 &20&98.6(5.3E-3)&98.0(6.3E-3)&95.6(9.2E-3)&0.50(1.4E-4)&0.48(1.5E-4)&0.47(2.4E-4)\\
\addlinespace
Toeplitz $\Sigma$\\
\multirow{2}{*}{$\alpha=1.5$}&5&97.2(7.4E-3)&96.6(8.1E-3)&94.0(1.1E-2)&12.53(1.9E-2)&13.70(2.8E-2)&16.82(4.7E-2)\\
 &20&99.8(2.0E-3)&99.0(4.4E-3)&95.8(9.0E-3)&106.4(8.9E-1)&60.94(4.4E-1)&18.65(2.9E-2)\\
\addlinespace
 \multirow{2}{*}{$\alpha=1.8$}&5&96.0(8.8E-3)&95.2(9.6E-3)&91.4(1.3E-2)&2.10(8.3E-4)&2.07(9.9E-4)&2.09(1.5E-3)\\
 &20&99.6(2.8E-3)&98.2(5.9E-3)&93.6(1.1E-2)&3.99(8.0E-3)&3.54(6.0E-3)&2.68(1.0E-3)\\
\addlinespace
\multirow{2}{*}{Gaussian}&5&94.6(1.0E-2)&92.6(1.2E-2)&87.0(1.5E-2)&0.29(4.8E-6)&0.27(8.9E-6)&0.25(1.4E-5)\\
 &20&99.0(4.4E-3)&99.4(3.5E-3)&95.8(9.0E-3)&0.64(2.2E-4)&0.62(2.4E-4)&0.62(4.6E-4)\\
\bottomrule
\end{tabular}}
\label{table:linear_CI}
\end{table}

\begin{enumerate}[(a)]
    \item[(a)] Across both identity and Toeplitz covariance structures, the coverage rates are generally close to the nominal 95\% target.
    \item[(b)] For dimension $d=5$, the coverage rates are relatively insensitive to the choice of subsample size when $\alpha=1.5$. In contrast, under Gaussian noise, coverage deteriorates when the subsample size is large (e.g., $n^{0.8}$).  
    \item[(c)] For dimension $d=20$, the coverage rates become sensitive to the choice of subsample size. Among the three options considered, $n^{0.8}$ yields the best performance, while the other two choices lead to conservative coverage. A plausible explanation is the limitation imposed by the total sample budget.  
    \item[(d)] In general, coverage rates tend to decrease as the subsample size increases. This pattern is consistent with the usual bias--variance trade-off in subsampling: smaller blocks may introduce approximation bias, whereas larger blocks reduce the effective number of blocks available for quantile estimation. Hence, there should exist an intermediate ``critical'' subsample size that balances these errors under a fixed sample budget.
    \item[(e)] For the average confidence region length $\widehat{q}\sqrt{\frac{\Tr(\Sigma_n)}{n}}$, we observe a decreasing trend as the stochastic gradients become lighter-tailed. Theoretically, the confidence region length is expected to scale as $n^{\frac{1}{\min\{\alpha,2\}}-1}$ up to logarithmic factors, which is consistent with our empirical findings.
\end{enumerate}

\FloatBarrier
\subsection{Logistic regression}
We consider the binary logistic regression model with heavy-tailed covariates studied by \citet{blanchet2024limit}:
\begin{align}
    \PB\left(y_i=1\mid x_i\right) = \frac{1}{1+\exp\left(-x_i^\top\theta^*\right)},\\
    \PB\left(y_i=-1\mid x_i\right) = \frac{\exp\left(-x_i^\top\theta^*\right)}{1+\exp\left(-x_i^\top\theta^*\right)},
\end{align}
where $\{(x_i,y_i)\}_{i=1}^n$ are i.i.d., $x_i\in\RB^d$, and $y_i\in\{-1,1\}$. For the distribution of $x_i$, we consider two cases:
\begin{enumerate}[(a)]
    \item Homogeneous tail indices: $x_i^{(1)},\ldots,x_i^{(d)}$ are i.i.d. symmetric Pareto random variables with tail index $\alpha$.
    \item Heterogeneous tail indices: $x_i^{(1)},\ldots,x_i^{(d)}$ are independent symmetric Pareto random variables with respective tail indices $\alpha_1,\ldots,\alpha_d$, where $\alpha_1=\alpha$, $\alpha_d=2.5$, and $\alpha_2,\ldots,\alpha_{d-1}\sim\operatorname{Uniform}(\alpha,2)$.
\end{enumerate}
Given a sample pair $(x_i,y_i)$, the stochastic gradient at $\theta$ is
\begin{align}
     g(\theta,x_i, y_i) = \frac{-y_i\exp(-y_ix_i^\top\theta)}{1+\exp(-y_ix_i^\top\theta)}x_i.
\end{align}
For this model, we set the sample size to $n=10^6$ and the dimension to $d=5$. The true parameter $\theta^*$ is generated from the standard Gaussian distribution $N(0,I_d)$. We consider three subsample sizes: $n^{0.6}$, $n^{0.7}$, and $n^{0.8}$. The nominal coverage level is 95\%, and we conduct 500 independent runs to estimate the empirical coverage rate. We evaluate Algorithm~\ref{alg: general} using the average coverage rate and average confidence interval length. Table~\ref{table:logistic_CI} reports the logistic regression results. Figure~\ref{fig: logistic_regression} displays the corresponding results for homogeneous covariates $x_i$ with $\alpha=1.5$ and $d=5$. We summarize the results below.

\begin{table}[t]
\caption{Logistic regression: average coverage rates and confidence interval lengths at the nominal $95\%$ coverage level. Standard errors are reported in parentheses.}
\resizebox{\textwidth}{!}{\begin{tabular}{lccccccc}
\toprule
 & $d$ & \multicolumn{3}{c}{Coverage Rate (\%)}  &  \multicolumn{3}{c}{Average Length ($\times 10^{-2}$)} \\
  &   & $n^{0.6}$&$n^{0.7}$&$n^{0.8}$ & $n^{0.6}$&$n^{0.7}$&$n^{0.8}$  \\
\midrule
\multicolumn{3}{l}{Homogeneous $x_i$}\\
$\alpha=1.5$&5&91.0(1.3E-2)&93.0(1.1E-2)&93.2(1.1E-2)&28.99(3.0E-2)&55.42(6.0E-2)&103.3(1.3E-1)\\
\addlinespace
$\alpha=1.8$&5&98.2(5.9E-3)&98.0(6.3E-3)&96.4(8.3E-3)&8.73(1.2E-3)&9.17(3.8E-3)&15.05(1.9E-2)\\
\addlinespace
Gaussian&5&99.8(2.0E-3)&97.4(7.1E-3)&90.4(1.3E-2)&0.92(2.0E-5)&0.74(2.9E-5)&0.63(3.7E-5)\\
\addlinespace
\multicolumn{6}{l}{Heterogeneous $x_i$, 4 heaviest coordinates}\\
$\alpha=1.5$&5&92.4(1.2E-2)&94.2(1.0E-2)&93.4(1.1E-2)&25.10(2.9E-2)&46.14(5.9E-2)&87.89(1.2E-1)\\
\addlinespace
$\alpha=1.8$&5&98.4(5.6E-3)&98.0(6.3E-3)&96.6(8.1E-3)&6.91(1.0E-3)&6.98(3.3E-3)&12.62(1.8E-2)\\
\addlinespace
\multicolumn{6}{l}{Heterogeneous $x_i$, 3 heaviest coordinates}\\
$\alpha=1.5$&5&92.8(1.2E-2)&94.0(1.1E-2)&93.6(1.1E-2)&23.12(2.7E-2)&41.77(5.5E-2)&79.29(1.2E-1)\\
\addlinespace
$\alpha=1.8$&5&98.4(5.6E-3)&98.2(5.9E-3)&97.0(7.6E-3)&6.72(9.8E-4)&6.74(3.2E-3)&12.16(1.7E-2)\\
\addlinespace
\multicolumn{6}{l}{Heterogeneous $x_i$, 2 heaviest coordinates}\\
$\alpha=1.5$&5&94.0(1.1E-2)&94.2(1.0E-2)&93.6(1.1E-2)&19.10(2.6E-2)&31.53(5.3E-2)&59.42(1.1E-1)\\
\addlinespace
$\alpha=1.8$&5&98.6(5.3E-3)&98.0(6.3E-3)&97.0(7.6E-3)&5.78(8.3E-4)&5.45(2.6E-3)&10.01(1.6E-2)\\
\addlinespace
\multicolumn{6}{l}{Heterogeneous $x_i$, 1 heaviest coordinate}\\
$\alpha=1.5$&5&95.2(9.6E-3)&95.2(9.6E-3)&94.6(1.0E-2)&24.74(5.3E-2)&15.85(2.6E-2)&48.48(1.1E-1)\\
\addlinespace
$\alpha=1.8$&5&98.6(5.3E-3)&98.0(6.3E-3)&97.0(7.6E-3)&5.37(7.7E-4)&4.88(2.3E-3)&9.13(1.5E-2)\\
\bottomrule
\end{tabular}}
\label{table:logistic_CI}
\end{table}

\begin{figure}[t]
    \centering
    \caption{Logistic regression with homogeneous covariates $x_i$: average coverage rate and confidence-interval length for $\alpha=1.5$ and $d=5$ under different subsample sizes.}
    \includegraphics[width=1.0\linewidth]{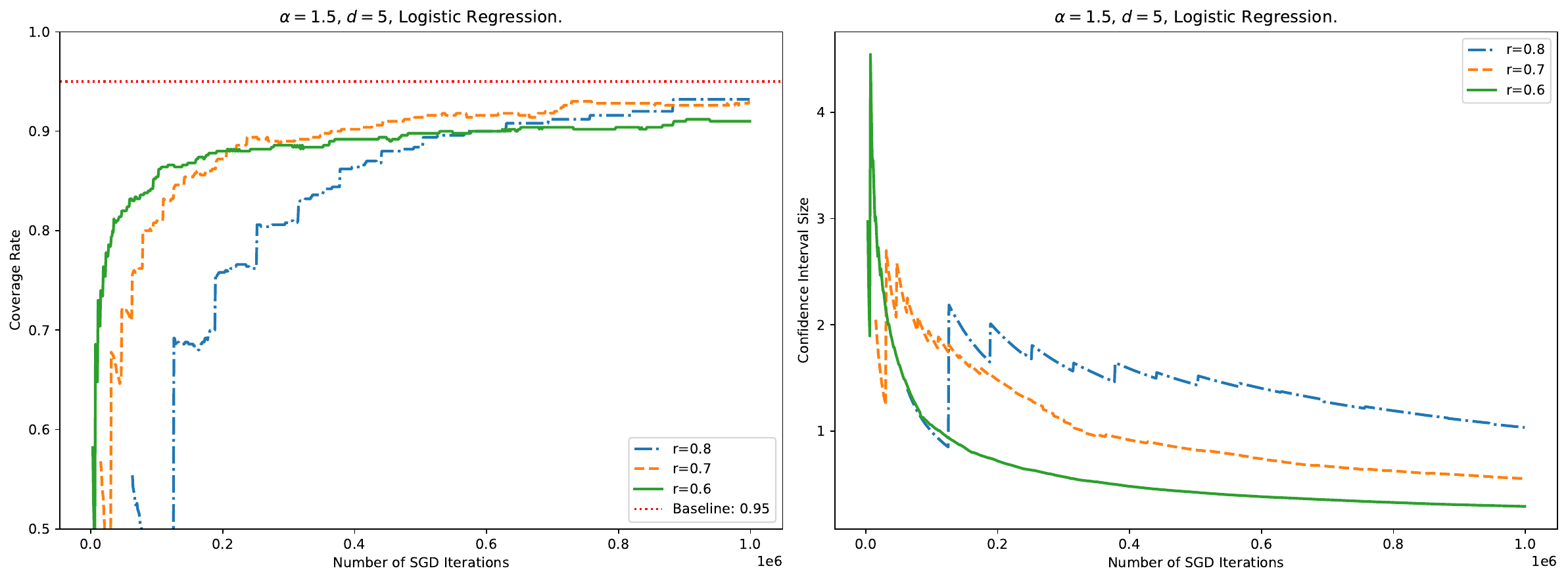}
    \label{fig: logistic_regression}
\end{figure}

\begin{enumerate}[(a)]
    \item[(a)] Across both homogeneous and heterogeneous covariates, the coverage rates are generally close to the nominal 95\% target.
    \item[(b)] Several observations are consistent with the linear regression results: (1) the coverage rates are still relatively insensitive to the choice of subsample size when $\alpha=1.5$; (2) the coverage rates tend to decrease as the subsample size increases; (3) the average confidence region length $\widehat{q}\sqrt{\frac{\Tr(\Sigma_n)}{n}}$  decreases as the stochastic gradients become lighter-tailed; (4) standard errors of the average confidence region length generally increase with the subsample size.  
    \item[(c)] In the heterogeneous case, the coverage rates remain consistent with those observed in the homogeneous setting.  
\end{enumerate}

\FloatBarrier

\section{Application to the CIFAR-10 dataset}
\label{sec: real}
In this section, we illustrate the proposed inference procedure on the CIFAR-10 dataset. This real-data experiment complements the simulation results in the previous section by considering an image-classification problem based on learned feature representations. The CIFAR-10 dataset is a standard benchmark for image classification. It consists of 60,000 color images of size $32\times 32$ from 10 object categories, with 50,000 training images and 10,000 test images. 

\begin{figure}[t]
\caption{Tail diagnostics for the $\ell_2$-norm of stochastic gradients in the AlexNet CIFAR-10 experiment. The stochastic gradients are evaluated near the reference parameter using mini-batches of size $64$. Left: empirical histogram of the gradient norms. Right: Hill estimates of the tail index $\alpha$.}
    \centering
    \begin{minipage}{0.48\textwidth}
        \centering
        \includegraphics[width=\textwidth]{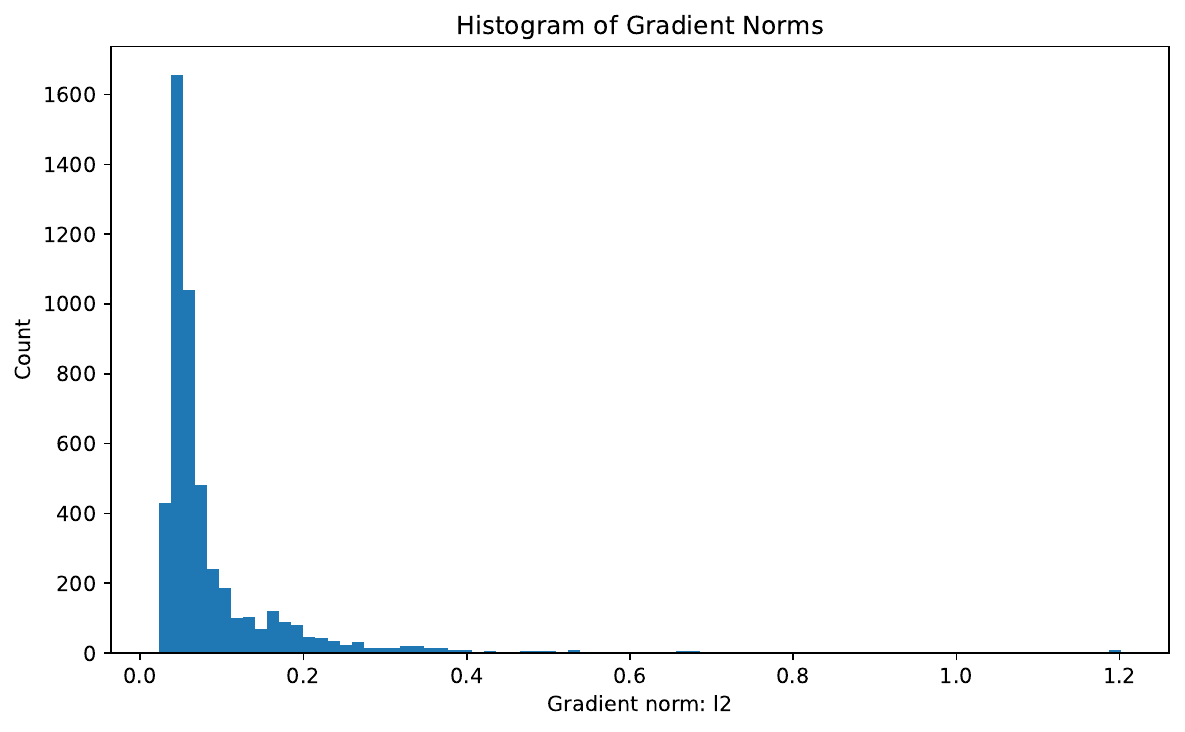}
    \end{minipage}
    \hfill
    \begin{minipage}{0.48\textwidth}
        \centering
        \includegraphics[width=\textwidth]{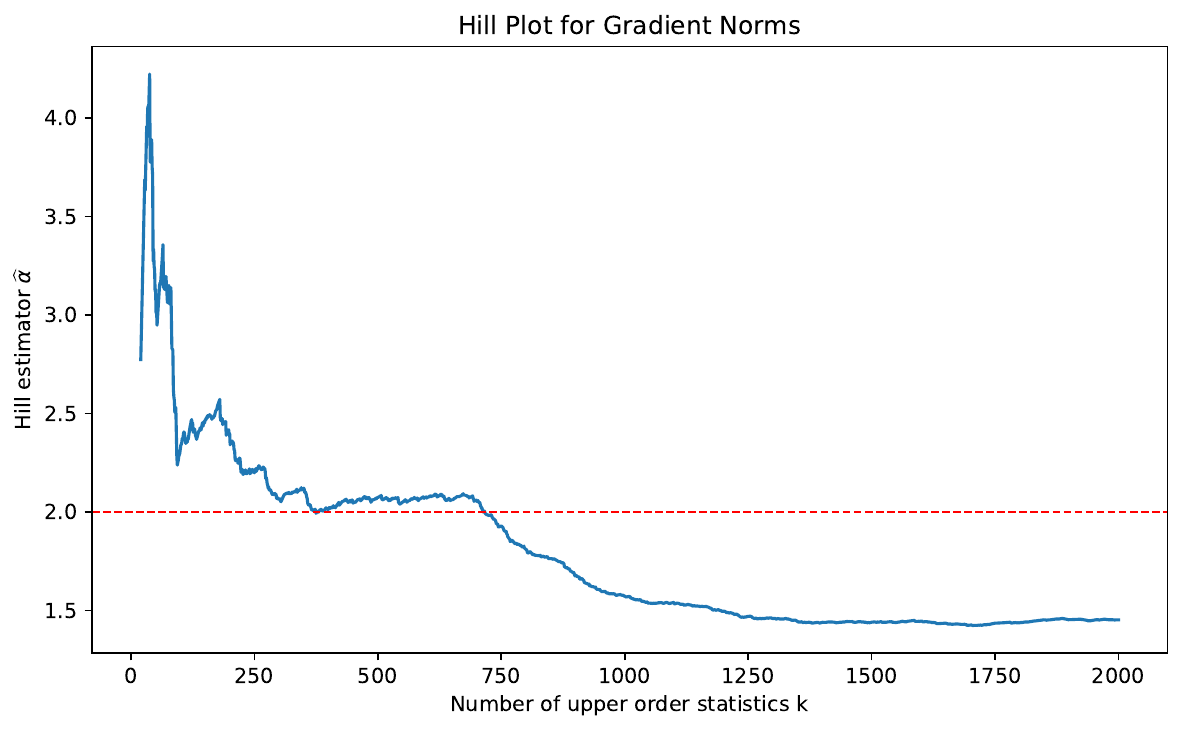}
    \end{minipage}
    \label{fig: small_tail}
\end{figure}

We first train an AlexNet convolutional neural network on the CIFAR-10 training set using stochastic gradient descent. This network is designed for $32\times 32$ inputs and uses a smaller fully connected representation, consistent with \citet{simsekli2019tail}. After training, we freeze the network and use the output of its penultimate layer, i.e., the layer before the final linear classifier, as the covariate vector. In our implementation, this produces a learned feature representation of dimension $448$. We then construct a binary classification task by selecting one CIFAR-10 class as the positive class and grouping the remaining classes as the negative class. We fit a logistic regression model to the extracted features. In Figure~\ref{fig: small_tail}, we show tail diagnostics for the $\ell_2$-norm of the stochastic gradient evaluated near the reference parameter in the AlexNet CIFAR-10 experiment. The mini-batch size is set to $64$, consistent with the experimental setup used in the subsequent inference results. The histogram on the left shows a highly right-skewed distribution with extreme values. The Hill plot is consistent with heavy-tailed behavior in this finite-sample experiment. In Section~\ref{sec: fullalex} of the Supplementary Material, we also consider a case in which the stochastic gradients exhibit lighter-tailed behavior under a different feature mapping. The result is qualitatively similar to that reported in this section.

To evaluate the coverage performance, we conduct $500$ independent runs. Since the true optimizer $\theta^*$ is unknown in this real-data setting, we compute a highly accurate reference solution using the L-BFGS algorithm and use it as a proxy for $\theta^*$. For each independent run, we construct coordinate-wise $95\%$ confidence intervals for all $448$ parameters. For each coordinate, the empirical coverage rate is then computed as the fraction of the $500$ confidence intervals that contain the corresponding coordinate of the L-BFGS reference solution. We provide additional details on the experimental setup and hyperparameter selection in Section~\ref{sec: setup} of the Supplementary Material. Here, we summarize the results by reporting the average coverage rate across all coordinates. In addition, we report the mean squared error (MSE) of the coordinate-wise coverage rates relative to the nominal level $95\%$, defined as
\begin{figure}[t]
    \centering
    \caption{Empirical reference coverage, coverage MSE, and fraction within the 95\% band in the CIFAR-10 experiment.}
    \includegraphics[width=0.9\linewidth]{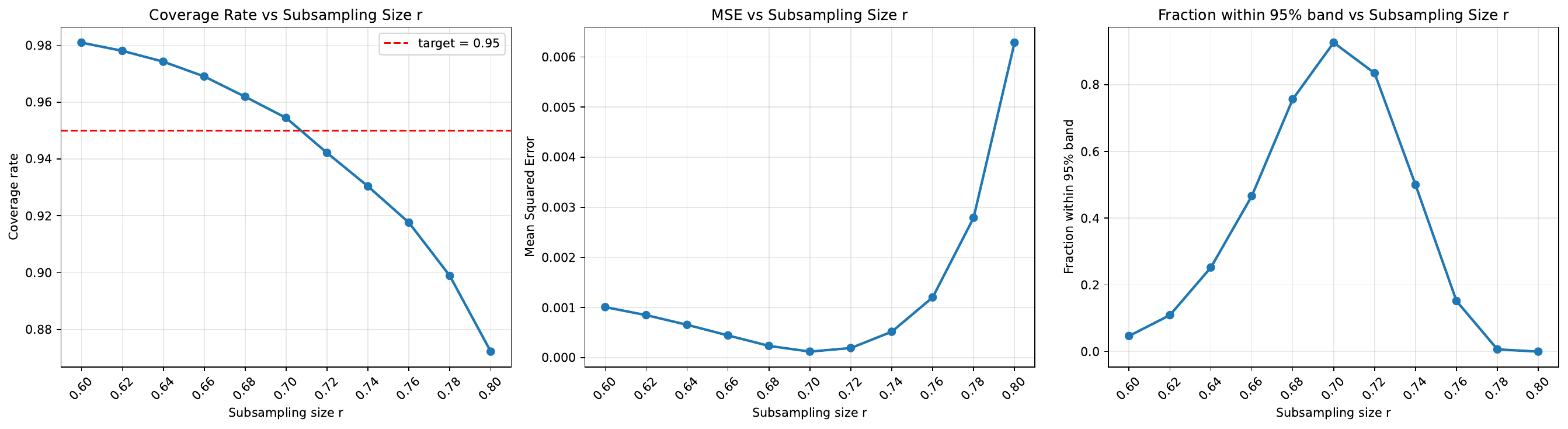}
    \label{fig: cifar10_sub_coverage_small}
\end{figure}

\begin{align}
    \frac{1}{448}\sum_{j=1}^{448}
    \left(\widehat{\mathrm{CR}}_j - 0.95\right)^2,
\end{align}
where $\widehat{\mathrm{CR}}_j$ denotes the empirical coverage rate for the $j$th coordinate. We also report a diagnostic measure called the ``Fraction within the $95\%$ band.'' If the true coverage probability is exactly $0.95$, then the Monte Carlo fluctuation of $\widehat{\mathrm{CR}}_j$ based on $500$ independent trials is approximately that of a binomial proportion with standard error $\sqrt{\frac{0.95(1-0.95)}{500}}$. We therefore compute the proportion of parameters whose empirical coverage rates fall within
\[
    0.95 \pm 1.96\sqrt{\frac{0.95(1-0.95)}{500}}.
\]

We first examine the sensitivity of the proposed method to the subsampling exponent. Specifically, we vary $r\in\{0.60,0.62,\ldots,0.78,0.80\}$. Figure~\ref{fig: cifar10_sub_coverage_small} reports the average coverage rate and coverage MSE under different choices of $r$. The average coverage rate decreases as $r$ increases: smaller values of $r$ tend to produce conservative confidence intervals, whereas larger values of $r$ lead to under-coverage. The coverage MSE exhibits a U-shaped pattern, decreasing as $r$ increases from $0.60$ to approximately $0.70$ and then increasing for larger values of $r$. Similarly, the fraction within the $95\%$ band is maximized at $r=0.7$. This pattern reflects the usual bias--variance trade-off in subsampling: smaller subsamples may yield conservative inference, whereas overly large subsamples leave fewer effective blocks for stable quantile estimation.

Based on Figure~\ref{fig: cifar10_sub_coverage_small}, we select $r=0.7$ for the proposed method in the following comparison. We compare Algorithm~\ref{alg: general} with the random scaling method \citep{lee2022fast}, which is effective for SGD in finite-variance settings. As shown in Table~\ref{table: cifar10_binary}, random scaling leads to noticeable under-coverage for nominal $95\%$ confidence intervals. In contrast, the proposed self-normalized subsampling method with $r=0.7$ achieves an average coverage rate closer to the nominal level and yields a substantially smaller coverage MSE. The proposed method also attains a substantially larger fraction within this band, indicating that its coordinate-wise coverage rates are more consistently aligned with the nominal level. Overall, these results suggest that the self-normalized subsampling procedure provides more reliable uncertainty quantification than the random scaling approach in this CIFAR-10 image-classification experiment.

\begin{table}[t]
\caption{Reference-solution coverage rate, mean squared error, and fraction within the $95\%$ band for $95\%$ confidence intervals in the CIFAR-10 experiment.}
\centering
\resizebox{\textwidth}{!}{\begin{tabular}{lccc}
\toprule
 & Average Coverage (\%) & Mean Squared Error ($\times 10^{-4}$) & Fraction within the $95\%$ band (\%) \\
\midrule
Random Scaling \citep{lee2022fast} & 92.05 & 10.43 & 20.54 \\
Algorithm~\ref{alg: general} ($r=0.7$) & 95.45 & 1.20 & 92.63 \\
\bottomrule
\end{tabular}}
\label{table: cifar10_binary}
\end{table}

\FloatBarrier

\section{Conclusion}
\label{sec: conc}
In this paper, we develop a statistical inference methodology for SGD with infinite-variance stochastic gradients. Our approach is based on a joint central limit theorem involving the Polyak--Ruppert averaged estimator and an empirical second-moment normalizer formed from the stochastic gradients. To address nuisance parameters in the limiting distribution, we construct a self-normalized statistic that cancels their influence. Building on the subsampling approach of \citet{romano1999subsampling}, we further design an online procedure for estimating the quantiles of the limiting distribution and constructing confidence regions for the true parameter. We evaluate the method in synthetic and CIFAR-10 classification experiments. The results suggest that self-normalized subsampling can provide stable uncertainty quantification in settings where standard finite-variance methods may exhibit under-coverage. Beyond these results, several important questions remain open. Progress on these questions would further broaden the scope of statistically principled uncertainty quantification for modern learning algorithms.

\paragraph*{\textbf{Choice of the subsampling level.}}

The subsampling exponent $r$ is a key hyperparameter in our inference procedure because it determines the block size $t_n=\lfloor n^r\rfloor$ used to approximate the distribution of the self-normalized statistic. This choice reflects a trade-off between CLT approximation bias and quantile estimation accuracy. The block size should be large enough for the block-level statistic to be close to its limiting distribution, but not so large that the effective number of blocks becomes too small to estimate quantiles accurately. An optimal choice would therefore require explicit rates for both error terms. In particular, controlling the approximation bias would likely require Berry--Esseen-type bounds, or related non-asymptotic bounds, for the self-normalized statistic constructed from SGD iterates. We leave the development of finite-sample results as an important direction for future research.

\paragraph*{\textbf{Extensions to other optimizers.}}
Although our analysis focuses on stochastic gradient descent, the main ideas underlying our inference framework may apply to a broader class of stochastic optimization algorithms, such as momentum-based methods and Adam \citep{polyak1964some,kingma2014adam}. Extending our methodology to such algorithms would require understanding the appropriate central limit theorems and corresponding normalization statistics. Once such results are established, the subsampling principle developed in this paper could be used in a fully data-driven manner.

\paragraph*{\textbf{Extensions to non-convex optimization.}}

Another important direction is to extend the current theory beyond strongly convex objectives. Many optimization problems arising in modern machine learning are non-convex, and an algorithm may converge to different local minima in such settings. One possible route is to pursue a local analysis around an isolated local minimizer. If the iterate sequence can be shown to enter and remain in a neighborhood where the objective is approximately strongly convex, then the asymptotic framework developed here may still apply after localization \citep{pelletier1998weak}. Under such a regime, one might construct confidence regions for the local minimizer, conditional on convergence to that basin of attraction. Moreover, rare but large gradients in the heavy-tailed regime further complicate the non-convex setting, potentially causing repeated escapes from shallow basins or transitions between metastable regions \citep{wang2021eliminating}. Understanding the interplay among heavy tails, non-convex geometry, and uncertainty quantification remains largely open.

\section*{Funding}
This work was supported by the Air Force Office of Scientific Research under award number FA9550-20-1-0397 and additional support is gratefully acknowledged from NSF 2229012, 2312204, and ONR 13983111.

\bibliographystyle{plainnat}
\bibliography{refer.bib}

@article{krizhevsky2012imagenet,
  title={{ImageNet} classification with deep convolutional neural networks},
  author={Krizhevsky, Alex and Sutskever, Ilya and Hinton, Geoffrey E},
  journal={Advances in Neural Information Processing Systems},
  volume={25},
  pages={1097--1105},
  year={2012}
}

@article{bai2016unified,
  title={A unified approach to self-normalized block sampling},
  author={Bai, Shuyang and Taqqu, Murad S and Zhang, Ting},
  journal={Stochastic Processes and Their Applications},
  volume={126},
  number={8},
  pages={2465--2493},
  year={2016},
  publisher={Elsevier}
}

@article{meerschaert1999sample,
  title={Sample covariance matrix for random vectors with heavy tails},
  author={Meerschaert, Mark M and Scheffler, Hans-Peter},
  journal={Journal of Theoretical Probability},
  volume={12},
  number={3},
  pages={821--838},
  year={1999},
  publisher={Springer}
}

@article{barsbey2021heavy,
  title={Heavy tails in {SGD} and compressibility of overparametrized neural networks},
  author={Barsbey, Melih and Sefidgaran, Milad and Erdogdu, Murat A and Richard, Gael and Simsekli, Umut},
  journal={Advances in Neural Information Processing Systems},
  volume={34},
  pages={29364--29378},
  year={2021}
}

@article{simsekli2020hausdorff,
  title={Hausdorff dimension, heavy tails, and generalization in neural networks},
  author={Simsekli, Umut and Sener, Ozan and Deligiannidis, George and Erdogdu, Murat A},
  journal={Advances in Neural Information Processing Systems},
  volume={33},
  pages={5138--5151},
  year={2020}
}

@inproceedings{mahoney2019traditional,
  title={Traditional and heavy tailed self regularization in neural network models},
  author={Mahoney, Michael and Martin, Charles},
  booktitle={International Conference on Machine Learning},
  pages={4284--4293},
  year={2019},
  organization={PMLR}
}

@article{damek2024analysing,
  title={Analysing heavy-tail properties of stochastic gradient descent by means of stochastic recurrence equations},
  author={Damek, Ewa and Mentemeier, Sebastian},
  journal={Journal of Applied Probability},
  volume={63},
  number={2},
  pages={459--483},
  year={2026},
  doi={10.1017/jpr.2025.10036}
}

@article{jiao2024emergence,
  title={Emergence of heavy tails in homogenized stochastic gradient descent},
  author={Jiao, Zhezhe and Keller-Ressel, Martin},
  journal={Advances in Neural Information Processing Systems},
  volume={37},
  pages={14066--14092},
  year={2024}
}

@article{schertzer2024stochastic,
  title={Stochastic differential equations models for least-squares stochastic gradient descent},
  author={Schertzer, Adrien and Pillaud-Vivien, Loucas},
  journal={Journal of Machine Learning Research},
  volume={27},
  number={75},
  pages={1--42},
  year={2026}
}

@article{logan1973limit,
  title={Limit distributions of self-normalized sums},
  author={Logan, Benjamin F and Mallows, CL and Rice, SO and Shepp, Larry A},
  journal={The Annals of Probability},
  volume={1},
  number={5},
  pages={788--809},
  year={1973},
  publisher={Institute of Mathematical Statistics}
}

@inproceedings{zhu2021constructing,
  title={On constructing confidence region for model parameters in stochastic gradient descent via batch means},
  author={Zhu, Yi and Dong, Jing},
  booktitle={2021 Winter Simulation Conference ({WSC})},
  pages={1--12},
  year={2021},
  organization={IEEE}
}

@article{zhong2023online,
  title={Online bootstrap inference with nonconvex stochastic gradient descent estimator},
  author={Zhong, Yanjie and Kuffner, Todd and Lahiri, Soumendra},
  journal={arXiv preprint arXiv:2306.02205},
  year={2023}
}

@article{blanchet2024limit,
  title={Limit theorems for stochastic gradient descent with infinite variance},
  author={Blanchet, Jose and Mijatovi{\'c}, Aleksandar and Yang, Wenhao},
  journal={The Annals of Applied Probability},
  year={2024},
  note={forthcoming},
  eprint={2410.16340},
  archivePrefix={arXiv}
}

@inproceedings{lee2022fast,
  title={Fast and robust online inference with stochastic gradient descent via random scaling},
  author={Lee, Sokbae and Liao, Yuan and Seo, Myung Hwan and Shin, Youngki},
  booktitle={Proceedings of the {AAAI} Conference on Artificial Intelligence},
  volume={36},
  pages={7381--7389},
  year={2022}
}

@article{zhu2024high,
  title={High confidence level inference is almost free using parallel stochastic optimization},
  author={Zhu, Wanrong and Lou, Zhipeng and Wei, Ziyang and Wu, Wei Biao},
  journal={arXiv preprint arXiv:2401.09346},
  year={2024}
}

@article{chen2020statistical,
  title={Statistical inference for model parameters in stochastic gradient descent},
  author={Chen, Xi and Lee, Jason D and Tong, Xin T and Zhang, Yichen},
  journal={The Annals of Statistics},
  volume={48},
  number={1},
  pages={251--273},
  year={2020},
  doi={10.1214/18-AOS1801}
}

@article{romano1999subsampling,
  title={Subsampling inference for the mean in the heavy-tailed case},
  author={Romano, Joseph P and Wolf, Michael},
  journal={Metrika},
  volume={50},
  number={1},
  pages={55--69},
  year={1999},
  publisher={Springer}
}

@book{harold1997stochastic,
  title={Stochastic approximation and recursive algorithms and applications},
  author={Kushner, Harold J. and Yin, G. George},
  series={Stochastic Modelling and Applied Probability},
  volume={35},
  publisher={Springer},
  address={New York},
  year={1997},
  doi={10.1007/978-1-4899-2696-8}
}

@article{csimcsekli2019heavy,
  title={On the heavy-tailed theory of stochastic gradient descent for deep neural networks},
  author={Simsekli, Umut and G{\"u}rb{\"u}zbalaban, Mert and Nguyen, Thanh Huy and Richard, Ga{\"e}l and Sagun, Levent},
  journal={arXiv preprint arXiv:1912.00018},
  year={2019}
}

@article{wang2021eliminating,
  title={Eliminating sharp minima from {SGD} with truncated heavy-tailed noise},
  author={Wang, Xingyu and Oh, Sewoong and Rhee, Chang-Han},
  journal={arXiv preprint arXiv:2102.04297},
  year={2021}
}

@article{sacks1958asymptotic,
  title={Asymptotic distribution of stochastic approximation procedures},
  author={Sacks, Jerome},
  journal={The Annals of Mathematical Statistics},
  volume={29},
  number={2},
  pages={373--405},
  year={1958},
  publisher={JSTOR}
}

@article{kingma2014adam,
  title={{Adam}: A method for stochastic optimization},
  author={Kingma, Diederik P and Ba, Jimmy},
  journal={arXiv preprint arXiv:1412.6980},
  year={2014}
}

@article{cutkosky2021high,
  title={High-probability bounds for non-convex stochastic optimization with heavy tails},
  author={Cutkosky, Ashok and Mehta, Harsh},
  journal={Advances in Neural Information Processing Systems},
  volume={34},
  pages={4883--4895},
  year={2021}
}

@article{liu2023stochastic,
  title={Stochastic Nonsmooth Convex Optimization with Heavy-Tailed Noises: High-Probability Bound, In-Expectation Rate and Initial Distance Adaptation},
  author={Liu, Zijian and Zhou, Zhengyuan},
  journal={arXiv preprint arXiv:2303.12277},
  year={2023}
}

@article{mou2022optimal,
  title={Optimal variance-reduced stochastic approximation in {Banach} spaces},
  author={Mou, Wenlong and Khamaru, Koulik and Wainwright, Martin J and Bartlett, Peter L and Jordan, Michael I},
  journal={arXiv preprint arXiv:2201.08518},
  year={2022}
}

@article{moulines2011non,
  title={Non-asymptotic analysis of stochastic approximation algorithms for machine learning},
  author={Moulines, Eric and Bach, Francis},
  journal={Advances in Neural Information Processing Systems},
  volume={24},
  pages={451--459},
  year={2011}
}

@article{pelletier1998almost,
  title={On the almost sure asymptotic behaviour of stochastic algorithms},
  author={Pelletier, Mariane},
  journal={Stochastic Processes and Their Applications},
  volume={78},
  number={2},
  pages={217--244},
  year={1998},
  publisher={Elsevier}
}

@article{bertsekas2000gradient,
  title={Gradient convergence in gradient methods with errors},
  author={Bertsekas, Dimitri P and Tsitsiklis, John N},
  journal={SIAM Journal on Optimization},
  volume={10},
  number={3},
  pages={627--642},
  year={2000},
  publisher={SIAM}
}

@incollection{robbins1971convergence,
  title={A convergence theorem for non negative almost supermartingales and some applications},
  author={Robbins, Herbert and Siegmund, David},
  booktitle={Optimizing Methods in Statistics},
  pages={233--257},
  year={1971},
  publisher={Elsevier}
}

@article{robbins1951stochastic,
  title={A stochastic approximation method},
  author={Robbins, Herbert and Monro, Sutton},
  journal={The Annals of Mathematical Statistics},
  volume={22},
  number={3},
  pages={400--407},
  year={1951},
  doi={10.1214/aoms/1177729586}
}

@article{polyak1964some,
  title={Some methods of speeding up the convergence of iteration methods},
  author={Polyak, Boris T},
  journal={{USSR} Computational Mathematics and Mathematical Physics},
  volume={4},
  number={5},
  pages={1--17},
  year={1964},
  doi={10.1016/0041-5553(64)90137-5}
}

@article{polyak1992acceleration,
  title={Acceleration of stochastic approximation by averaging},
  author={Polyak, Boris T and Juditsky, Anatoli B},
  journal={{SIAM} Journal on Control and Optimization},
  volume={30},
  number={4},
  pages={838--855},
  year={1992},
  publisher={SIAM}
}

@article{li1994almost,
  title={Almost sure convergence of stochastic approximation procedures},
  author={Li, Gang},
  journal={Statistica Sinica},
  volume={4},
  number={1},
  pages={361--372},
  year={1994}
}

@article{goodsell1976almost,
  title={Almost sure convergence for the {Robbins--Monro} process},
  author={Goodsell, C. A. and Hanson, D. L.},
  journal={The Annals of Probability},
  volume={4},
  number={6},
  pages={890--901},
  year={1976}
}

@article{krasulina1969stochastic,
  title={On stochastic approximation processes with infinite variance},
  author={Krasulina, Tatiana Pavlovna},
  journal={Theory of Probability \& Its Applications},
  volume={14},
  number={3},
  pages={522--526},
  year={1969},
  publisher={SIAM}
}

@inproceedings{simsekli2019tail,
  title={A tail-index analysis of stochastic gradient noise in deep neural networks},
  author={Simsekli, Umut and Sagun, Levent and Gurbuzbalaban, Mert},
  booktitle={International Conference on Machine Learning},
  pages={5827--5837},
  year={2019},
  organization={PMLR}
}

@inproceedings{hodgkinson2021multiplicative,
  title={Multiplicative noise and heavy tails in stochastic optimization},
  author={Hodgkinson, Liam and Mahoney, Michael},
  booktitle={International Conference on Machine Learning},
  pages={4262--4274},
  year={2021},
  organization={PMLR}
}

@inproceedings{gurbuzbalaban2021heavy,
  title={The heavy-tail phenomenon in {SGD}},
  author={Gurbuzbalaban, Mert and Simsekli, Umut and Zhu, Lingjiong},
  booktitle={International Conference on Machine Learning},
  pages={3964--3975},
  year={2021},
  organization={PMLR}
}

@article{pelletier1998weak,
  title={Weak convergence rates for stochastic approximation with application to multiple targets and simulated annealing},
  author={Pelletier, Mariane},
  journal={The Annals of Applied Probability},
  volume={8},
  number={1},
  pages={10--44},
  year={1998}
}

@book{resnick2007heavy,
  title={Heavy-tail phenomena: probabilistic and statistical modeling},
  author={Resnick, Sidney I},
  year={2007},
  publisher={Springer Science \& Business Media}
}

@article{wang2021convergence,
  title={Convergence rates of stochastic gradient descent under infinite noise variance},
  author={Wang, Hongjian and Gurbuzbalaban, Mert and Zhu, Lingjiong and Simsekli, Umut and Erdogdu, Murat A},
  journal={Advances in Neural Information Processing Systems},
  volume={34},
  pages={18866--18877},
  year={2021}
}

\clearpage
\appendix
\section*{Supplementary Material}
\addcontentsline{toc}{section}{Supplementary Material}

% Retain the numbering convention of the standalone supplementary file.
\mathtoolsset{showonlyrefs=false}
\setcounter{thm}{0}
\setcounter{lem}{0}
\setcounter{cor}{0}
\setcounter{prop}{0}
\setcounter{rem}{0}
\setcounter{example}{0}
\setcounter{asmp}{0}
\setcounter{defn}{0}
\setcounter{oracle}{0}
\setcounter{fact}{0}
\setcounter{claim}{0}
\setcounter{conj}{0}
\setcounter{condition}{0}
\renewcommand{\thethm}{S\arabic{thm}}
\renewcommand{\thelem}{S\arabic{lem}}
\renewcommand{\thecor}{S\arabic{cor}}
\renewcommand{\theprop}{S\arabic{prop}}
\renewcommand{\therem}{S\arabic{rem}}
\renewcommand{\theexample}{S\arabic{example}}
\renewcommand{\theasmp}{S\arabic{asmp}}
\renewcommand{\thedefn}{S\arabic{defn}}
\renewcommand{\theoracle}{S\arabic{oracle}}
\renewcommand{\thefact}{S\arabic{fact}}
\renewcommand{\theclaim}{S\arabic{claim}}
\renewcommand{\theconj}{S\arabic{conj}}
\renewcommand{\thecondition}{S\arabic{condition}}
\numberwithin{equation}{section}
\numberwithin{figure}{section}
\numberwithin{table}{section}

\section{Generalized central limit theorem (GCLT)}
\label{apd: gclt}
For i.i.d. random vectors $X_1,\ldots,X_n$ with $\EB[\|X_1\|]<+\infty$ and finite covariance $\EB[X_1X_1^\top]=\Sigma$, the classical CLT gives the following weak convergence:
\begin{align}
\label{eq: clt_gau}
    \frac{1}{\sqrt{n}}\sum_{i=1}^n (X_i-\EB[X_i])\tod N\left(0,\Sigma\right).
\end{align}
In the setting considered in this paper, some components of $X_1$ may not have finite variance, so the weak convergence in equation~\eqref{eq: clt_gau} no longer holds. To establish a limit theorem in the infinite-variance case, we instead impose the following standard multivariate regular variation condition; see \citet{resnick2007heavy}.
\begin{asmp}[Multivariate regular variation]
\label{asmp: mrv}
    There exist $\alpha\in(1,2)$ and a slowly varying function
    $b_0:\mathbb{R}_+\to\mathbb{R}_+$ such that
    \begin{align}
        \label{eq: norm}
        \PB\left(\|X\|> r\right)=r^{-\alpha}b_0(r).
    \end{align}
    Let $a(x)=\inf\{r>0:\PB(\|X\|>r)\le x^{-1}\}$ and define the slowly
    varying normalization $b_1(x)=x^{1/\alpha}/a(x)$. This generalized-inverse
    convention also covers distributions with atoms.
    Then there exists a Radon measure $\nu(\cdot)$ on
    $\bar{\mathbb{R}}^d\setminus\{0\}$ such that
    \begin{align}
        \label{eq: radon}
        x\PB\left(x^{-1/\alpha}b_1(x)X\in\cdot\right)
        \overset{v}{\to}\nu(\cdot),\qquad x\to+\infty.
    \end{align}
    The measure $\nu$ is homogeneous of order $-\alpha$.
\end{asmp}
Assumption~\ref{asmp: mrv} describes multivariate regular variation directly
through the radial tail $b_0$ and the limiting Radon measure $\nu$. The
normalization $n^{-1/\alpha}b_1(n)$ in \eqref{eq: radon} determines the
scaling rate in the generalized CLT.
\begin{thm}[Generalized CLT \citep{resnick2007heavy}]
\label{thm: stable_clt}
    For i.i.d. random vectors $X_1,X_2,\ldots$ satisfying Assumption~\ref{asmp: mrv}, the following weak convergence holds on $D_{J_1}([0,1],\RB^d)$:
    \begin{align}
        n^{-\frac{1}{\alpha}}b_1(n)\sum_{i=1}^{\lfloor nt\rfloor}\left(X_i - \EB[X_i]\right)\Rightarrow L_t,
    \end{align}
    where $L_t$ is a L\'evy process with characteristics $(0,\nu(\cdot),\gamma)$ and $\gamma=-\int_{\|x\|>1}x\nu(dx)$. That is, for any $u\in\mathbb R^d$, the characteristic function of $L_t$ is
    \begin{align}
        \EB\left[\exp\left(iu^\top L_t\right)\right]
        =\exp\left(t\left\{iu^\top\gamma+\int\left(e^{iu^\top x}-1-iu^\top x\mathbbm{1}(\|x\|\le 1)\right)\nu(dx)\right\}\right).
    \end{align}
\end{thm}
In the infinite-variance case, Theorem~\ref{thm: stable_clt} characterizes the weak convergence of normalized sums of i.i.d. random vectors. This generalized CLT depends on several structural parameters: the tail index $\alpha$, the slowly varying function $b_1(\cdot)$, and the Radon measure $\nu(\cdot)$. Together, these parameters determine the precise form of the limiting law. In contrast to the classical result in equation~\eqref{eq: clt_gau}, the scaling factor changes from $n^{-1/2}$ to $n^{-1/\alpha}b_1(n)$, and the limiting distribution is heavy-tailed rather than Gaussian. In particular, when $t=1$, the random vector $L_1$ has heavy-tailed marginals, and its radial component $\|L_1\|$ follows an $\alpha$-regularly varying distribution. One example of $L_1$ is an elliptically contoured stable distribution.
\begin{defn}[Scale function]
\label{defn: scale}
    We define the directional scale of an $\alpha$-stable distribution, for $u\in\RB^d$, by
    \begin{align}
        \sigma_{\alpha}(Z;u)=\int_{\mathbb{S}^{d-1}}\left|s^\top u\right|^\alpha\Gamma(ds),
    \end{align}
    where $\Gamma(\cdot)$ is the spectral measure of $Z$ on the unit sphere.
\end{defn}
This is the $\alpha$th-power convention for scale (that is, no
$1/\alpha$ root is taken), and it is defined homogeneously for every
$u\in\mathbb R^d$.
The elliptically contoured multivariate stable distribution is a symmetric subclass of the general multivariate $\alpha$-stable family. When $\alpha=2$, it reduces to the multivariate Gaussian distribution.

\section{Confidence regions for artificially injected noise}
\label{sec: q_aware}
% \begin{thm}
% \label{thm: pa_fclt}
%      In Skorokhod $M_1$ topology, the joint weak convergence holds:
%     \begin{align}
%         n^{1-\frac{1}{\alpha}}b_1(n)\left(\bar{\theta}_{\lfloor n\cdot\rfloor}-\theta^*\right)&\tod -H^{-1}L_1(\cdot).
%     \end{align}
% \end{thm}
In perturbed gradient descent, practitioners choose the artificial noise $\xi_{\text{art}}$, so its distribution is known. We consider independent symmetric Pareto coordinates. Intrinsic randomness may still be present, but its tail is required to be lighter than that of the injected noise.
\begin{asmp}
\label{asmp: pgd}
    For the stochastic gradient $g(\theta,\xi)$ in equation~\eqref{eq: decomp_sg}, the pairs $(\xi_{\mathrm{int},j},\xi_{\mathrm{art},j})$ are i.i.d. across iterations, with $\xi_{\mathrm{int},j}$ independent of $\xi_{\mathrm{art},j}$. The coordinates of $\xi_{\text{art}}\in\RB^d$ are i.i.d. symmetric Pareto random variables with index $\alpha\in(1,2)$ and scale $\lambda>0$, i.e., for $i\in[d]$,
    \begin{align}
        \frac{d}{dt}\PB\left(\xi^{(i)}_{\text{art}}\le t\right) = \frac{\alpha \lambda^{\alpha}}{2(\lambda+|t|)^{\alpha+1}}.
    \end{align}
    Moreover, $\lim_{t\to+\infty}t^{\alpha}\PB\left(\|g(\theta, \xi_{\text{int}})\|>t\right)=0$, so the artificial noise dominates the intrinsic randomness.
\end{asmp}
In this case, the slowly varying normalization in Theorem~\ref{thm: pa_limit} becomes an explicit constant, and the limit distribution is known except for the Hessian matrix $H$. Thus, if a consistent Hessian estimator $\widehat{H}_n \overset{p}{\to} H$ is available, it can be plugged into the confidence-region construction. We consider the following estimator:
\begin{align}
\label{eq: hessian}
    \widehat{H}_n=\frac{1}{n}\sum_{k=1}^n\nabla g(\theta_{k-1},\xi_k).
\end{align}
We also impose a smoothness condition on the stochastic Hessian matrix.
\begin{asmp}
\label{asmp: hessian}
    We denote the stochastic Hessian matrix by $H(\theta,\xi)=\nabla g(\theta,\xi)$ and assume that, for any $\theta_1,\theta_2$,
    \begin{align}
        \|H(\theta_1,\xi)-H(\theta_2,\xi)\|\le C_{H,\text{Lip}}\|\theta_1-\theta_2\| \text{ a.s.}
    \end{align}
    We also assume $\EB\|H(\theta^*,\xi)\|<\infty$ and
    $\EB[H(\theta^*,\xi)]=H$.
\end{asmp}
\begin{thm}
\label{thm: hessian_consistency}
    Suppose $\eta_n=cn^{-\rho}$ with $\rho\in(q^{-1},1)$. Under Assumptions~\ref{asmp: tail_norm}, \ref{asmp: ppd}, \ref{asmp: lip_error}, \ref{asmp: talyor_error}, \ref{asmp: pgd}, and \ref{asmp: hessian}, the estimator $\widehat{H}_n$ in equation~\eqref{eq: hessian} is consistent; that is, $\widehat{H}_n\overset{p}{\to}H$.
    Additionally, if there exists $\alpha_H>1$ such that
    \begin{align}
        \EB\left\|H(\theta^*,\xi)-H\right\|^{\alpha_H}\le C_H,
    \end{align}
    then, for any $p\in(1,\alpha)$,
    \begin{align}
        \EB\left\|\widehat{H}_n-H\right\|\lesssim
        \frac{1}{n^{1-\frac{1}{\min\{\alpha_H,2\}}}}
        +\frac{1}{n^{\rho\left(1-\frac{1}{p}\right)}}.
    \end{align}
\end{thm}
Under the preceding conditions, the Hessian estimator $\widehat{H}_n$ in equation~\eqref{eq: hessian} is consistent for $H$. The next proposition identifies the stable law generated by the injected noise and fixes its normalization explicitly.
\begin{prop}
\label{prop: charac}
    Under Assumption~\ref{asmp: pgd},
    \begin{align}
        n^{-\frac{1}{\alpha}}\sum_{j=1}^n g(\theta^*,\xi_j)\tod L_{\alpha,\lambda},
    \end{align}
    where $L_{\alpha,\lambda}$ has independent symmetric $\alpha$-stable coordinates and characteristic function
    \begin{align}
        \EB\left[\exp\left(iu^\top L_{\alpha,\lambda}\right)\right]
        =\exp\left(-\kappa_{\alpha,\lambda}\|u\|_{\alpha}^{\alpha}\right).
    \end{align}
    Consequently, under the conditions of Theorem~\ref{thm: pa_limit},
    \begin{align}
        n^{1-\frac{1}{\alpha}}(\bar\theta_n-\theta^*)
        \tod -H^{-1}L_{\alpha,\lambda}
        \overset{d}{=}H^{-1}L_{\alpha,\lambda},
        \label{eq: aware_clt}
    \end{align}
    and
    \begin{align}
        \EB\left[\exp\left(iu^\top\left(-H^{-1}L_{\alpha,\lambda}\right)\right)\right]
        =\exp\left(-\kappa_{\alpha,\lambda}\|H^{-1}u\|_{\alpha}^{\alpha}\right),
        \qquad
        \kappa_{\alpha,\lambda}
        :=\frac{\pi\lambda^\alpha}{2\Gamma_{\mathrm E}(\alpha)\sin(\pi\alpha/2)},
    \end{align}
    where $\Gamma_{\mathrm E}$ denotes Euler's gamma function.
\end{prop}
Thus, because $\widehat{H}_n^{-1}$ consistently estimates $H^{-1}$, the quantiles of $\widehat{H}_n^{-1}L_{\alpha,\lambda}$ can be calculated or simulated. For a continuous homogeneous function $\varphi:\mathbb{R}^d\to\mathbb{R}$ of degree 1, the confidence region can be constructed as
\begin{align}
\label{eq: cr_aware}
    \widehat{\RM}_n^\dagger(\delta;\varphi)=\left\{\theta\in\RB^d\left|n^{1-\frac{1}{\alpha}}\cdot \left|\varphi\left(\bar{\theta}_n-\theta\right)\right|\le q_n^\dagger(\delta;\varphi)\right.\right\},
\end{align}
where
\begin{align}
    q_n^\dagger(\delta;\varphi)=\inf\left\{t\left|\PB\left(\left|\varphi(\widehat{H}_n^{-1}L_{\alpha,\lambda})\right|>t\mid\widehat{H}_n\right)\le\delta\right.\right\}.
\end{align}
It can be verified that the confidence region $\widehat{\RM}_n^\dagger(\delta;\varphi)$ is asymptotically valid.
\begin{thm}
\label{thm: cr_aware}
    Suppose the distribution function of $|\varphi(H^{-1}L_{\alpha,\lambda})|$ is continuous and strictly increasing at its $(1-\delta)$-quantile. Under the conditions of Theorems~\ref{thm: pa_limit} and \ref{thm: hessian_consistency} and Proposition~\ref{prop: charac}, the confidence region $\widehat{\RM}_n^\dagger(\delta;\varphi)$ in equation~\eqref{eq: cr_aware} satisfies
    \begin{align}
        \lim_{n\to+\infty}\PB\left(\theta^*\in\widehat{\RM}_n^\dagger(\delta;\varphi)\right)=1-\delta.
    \end{align}
\end{thm}
While this artificially injected noise approach provides a conceptually simple route for inference, it also has several limitations. First, it requires the injected noise to dominate the intrinsic stochastic gradient noise so that the asymptotic distribution is determined solely by the artificial perturbation. In practice, however, the tail index of the intrinsic stochastic gradient noise is typically unknown, making it difficult to guarantee such dominance. Second, when the injected noise dominates the learning dynamics, the resulting confidence regions primarily reflect the variability of the artificial perturbation rather than that of the intrinsic stochastic gradient noise. Consequently, these regions may become overly conservative and fail to capture useful information about the curvature of the loss landscape. These considerations motivate the development of an inference procedure that does not rely on artificially injected noise.

\subsection{Proof of Theorem~\ref{thm: hessian_consistency}}
We decompose the error $\widehat{H}_n-H$ as
\begin{align}
    \widehat{H}_n-H&=\frac{1}{n}\sum_{k=1}^n \left(H(\theta^*,\xi_k)-H\right)+\frac{1}{n}\sum_{k=1}^n\left(H(\theta_{k-1},\xi_k)-H(\theta^*,\xi_k)\right)\\
    &:=\Delta_1+\Delta_2.
\end{align}
By the strong law of large numbers,
\begin{align}
\label{eq: h_1}
    \Delta_1\overset{a.s.}{\to}0.
\end{align}
Moreover, Assumption~\ref{asmp: hessian} gives
\begin{align}
    \label{eq: h_3}
    \EB\|\Delta_2\|&\le\frac{1}{n}\sum_{k=1}^n\EB\left\|H(\theta_{k-1},\xi_k)-H(\theta^*,\xi_k)\right\|\\
    &\le \frac{C_{H,\text{Lip}}}{n}\sum_{k=1}^n\EB\|\theta_{k-1}-\theta^*\|\\
    &=\OM\left(n^{-\rho\frac{p-1}{p}}\right),
\end{align}
where $p\in(1,\alpha)$ and the last bound follows from Lemma~\ref{lem: final_error}. Therefore,
\begin{align}
\label{eq: h_2}
    \Delta_2\overset{p}{\to}0.
\end{align}
Combining equations~\eqref{eq: h_1} and \eqref{eq: h_2}, we conclude that
\begin{align}
    \widehat{H}_n\overset{p}{\to}H.
\end{align}
Additionally, suppose there exists $\alpha_H>1$ such that
\begin{align}
    \EB\left\|H(\theta^*,\xi)-H\right\|^{\alpha_H}\le C_H.
\end{align}
Then Lyapunov's inequality gives a finite moment of order $\min\{\alpha_H,2\}$.
Applying Lemma~\ref{lem: moment} to the vectorized matrix martingale and using
equivalence of matrix norms in finite dimension yields
\begin{align}
    \label{eq: h_4}
    \EB\left\|\Delta_1\right\|
    \lesssim n^{-1+\frac{1}{\min\{\alpha_H,2\}}}.
\end{align}
Thus, combining equations~\eqref{eq: h_3} and \eqref{eq: h_4}, we obtain
\begin{align}
    \EB\left\|\widehat{H}_n-H\right\|\lesssim
    \frac{1}{n^{1-\frac{1}{\min\{\alpha_H,2\}}}}
    +\frac{1}{n^{\rho\left(1-\frac{1}{p}\right)}}.
\end{align}
For $\alpha_H\ge2$, this gives the $n^{-1/2}$ rate for $\Delta_1$; higher moments provide stronger integrability but cannot improve this generic mean-estimation rate.

\subsection{Proof of Proposition~\ref{prop: charac}}
Write
\[
    g(\theta^*,\xi_j)=\xi_{\mathrm{art},j}+Y_j,
    \qquad
    Y_j:=g(\theta^*,\xi_{\mathrm{int},j}).
\]
The random vectors $Y_j$ are i.i.d. and centered: $\EB Y_1=0$ because
$\EB g(\theta^*,\xi)=\nabla\ell(\theta^*)=0$ and the integrable symmetric
noise $\xi_{\mathrm{art}}$ has mean zero. Let $a_n=n^{1/\alpha}$ and
$\varepsilon(t)=t^\alpha\PB(\|Y_1\|>t)$. Assumption~\ref{asmp: pgd}
gives $\varepsilon(t)\to0$. First,
\[
    \PB\left(\max_{j\le n}\|Y_j\|>a_n\right)
    \le n\PB(\|Y_1\|>a_n)=\varepsilon(a_n)\to0.
\]
For any fixed $T>0$,
\begin{align}
    \EB\left[\|Y_1\|^2\mathbbm{1}(\|Y_1\|\le a_n)\right]
    &\le \int_0^{a_n}2t\PB(\|Y_1\|>t)\,dt \notag\\
    &\le T^2+\frac{2\sup_{t\ge T}\varepsilon(t)}{2-\alpha}a_n^{2-\alpha}
    =o(a_n^{2-\alpha}).
\end{align}
Independence therefore implies
\begin{align}
    \EB\left\|\frac{1}{a_n}\sum_{j=1}^n
    \left(Y_j\mathbbm{1}(\|Y_j\|\le a_n)
    -\EB[Y_1\mathbbm{1}(\|Y_1\|\le a_n)]\right)\right\|^2
    =o(na_n^{-\alpha})=o(1).
\end{align}
Moreover, since $\EB Y_1=0$,
\begin{align}
    \frac{n}{a_n}\left\|\EB[Y_1\mathbbm{1}(\|Y_1\|\le a_n)]\right\|
    &\le \frac{n}{a_n}\EB[\|Y_1\|\mathbbm{1}(\|Y_1\|>a_n)]\\
    &\le \frac{\alpha}{\alpha-1}\sup_{t\ge a_n}\varepsilon(t)\to0.
\end{align}
Combining these bounds gives
\begin{align}
    n^{-1/\alpha}\sum_{j=1}^nY_j\overset{p}{\to}0.
    \label{eq: intrinsic_negligible}
\end{align}

Next let $\chi(s)=\EB[\exp(is\xi_{\mathrm{art}}^{(1)})]$. Symmetry and
Assumption~\ref{asmp: pgd} give, for $s>0$,
\begin{align}
    1-\chi(s)
    &=\int_0^\infty(1-\cos(sr))
      \frac{\alpha\lambda^\alpha}{(\lambda+r)^{\alpha+1}}\,dr\\
    &=\alpha\lambda^\alpha s^\alpha
      \int_0^\infty\frac{1-\cos y}{(\lambda s+y)^{\alpha+1}}\,dy.
\end{align}
Dominated convergence applies because $(1-\cos y)y^{-\alpha-1}$ is
integrable on $(0,\infty)$, and hence
\begin{align}
    1-\chi(s)=\alpha\lambda^\alpha I_\alpha|s|^\alpha(1+o(1)),
    \qquad
    I_\alpha:=\int_0^\infty(1-\cos y)y^{-\alpha-1}\,dy.
\end{align}
Integration by parts and the Mellin integral for the sine function yield
\begin{align}
    I_\alpha
    &=\frac{1}{\alpha}\int_0^\infty y^{-\alpha}\sin y\,dy
      =\frac{\Gamma_{\mathrm E}(1-\alpha)\cos(\pi\alpha/2)}{\alpha},\\
    \alpha\lambda^\alpha I_\alpha
    &=\frac{\pi\lambda^\alpha}
      {2\Gamma_{\mathrm E}(\alpha)\sin(\pi\alpha/2)}
      =\kappa_{\alpha,\lambda}.
\end{align}
It follows that, for every $t\in\RB$,
\begin{align}
    \chi(tn^{-1/\alpha})^n
    \longrightarrow \exp(-\kappa_{\alpha,\lambda}|t|^\alpha).
\end{align}
Because the coordinates of $\xi_{\mathrm{art}}$ are independent, for every
$u\in\RB^d$,
\begin{align}
    \EB\exp\left(iu^\top n^{-1/\alpha}
        \sum_{j=1}^n\xi_{\mathrm{art},j}\right)
    &=\prod_{k=1}^d\chi(u_kn^{-1/\alpha})^n\\
    &\longrightarrow
      \exp(-\kappa_{\alpha,\lambda}\|u\|_\alpha^\alpha).
\end{align}
L\'evy's continuity theorem together with \eqref{eq: intrinsic_negligible}
proves the first assertion.

Finally, Assumption~\ref{asmp: pgd} implies
$\PB(\|g(\theta^*,\xi)\|>t)\sim d\lambda^\alpha t^{-\alpha}$, so the
normalizer $b_1(n)$ in Theorem~\ref{thm: pa_limit} converges to
$(d\lambda^\alpha)^{-1/\alpha}\in(0,\infty)$. Dividing the leading-term
expansion in the proof of that theorem by $b_1(n)$ gives
\begin{align}
    n^{1-\frac{1}{\alpha}}(\bar\theta_n-\theta^*)
    =-H^{-1}n^{-\frac{1}{\alpha}}
      \sum_{j=1}^n g(\theta^*,\xi_j)+o_p(1).
\end{align}
Slutsky's theorem proves \eqref{eq: aware_clt}. Since
$L_{\alpha,\lambda}$ is symmetric, its reflection has the same law. Finally,
$H^{-1}$ is symmetric, and replacing $u$ by $H^{-1}u$ in the characteristic
function above gives the stated formula.\hfill\ensuremath{\qedsymbol}

\subsection{Proof of Theorem~\ref{thm: cr_aware}}
Let
\begin{align}
    q^\dagger(\delta;\varphi)
    =\inf\left\{t:\PB\left(\left|\varphi(H^{-1}L_{\alpha,\lambda})\right|\le t\right)\ge1-\delta\right\}.
\end{align}
Since $H$ is positive definite and $\widehat H_n\overset{p}{\to}H$,
$\widehat H_n$ is invertible with probability tending to one and
$\widehat H_n^{-1}\overset{p}{\to}H^{-1}$. For any matrix $A$, define
\begin{align}
    G_A(t;\varphi)=\PB\left(\left|\varphi(AL_{\alpha,\lambda})\right|\le t\right).
\end{align}
The continuous mapping theorem and bounded convergence imply that, at every
continuity point $t$ of $G_{H^{-1}}(\cdot;\varphi)$,
\begin{align}
    G_{\widehat H_n^{-1}}(t;\varphi)\overset{p}{\to}G_{H^{-1}}(t;\varphi).
\end{align}
Strict increase at $q^\dagger(\delta;\varphi)$ therefore gives, for every
$\varepsilon>0$,
\begin{align}
    G_{H^{-1}}(q^\dagger-\varepsilon;\varphi)<1-\delta
    <G_{H^{-1}}(q^\dagger+\varepsilon;\varphi),
\end{align}
and hence
\begin{align}
    q_n^\dagger(\delta;\varphi)\overset{p}{\to}q^\dagger(\delta;\varphi).
\end{align}
Proposition~\ref{prop: charac}, Slutsky's theorem, and continuity of the limit
distribution at $q^\dagger(\delta;\varphi)$ now yield
\begin{align}
    \PB\left(\theta^*\in\widehat{\RM}_n^\dagger(\delta;\varphi)\right)
    \longrightarrow
    \PB\left(\left|\varphi(H^{-1}L_{\alpha,\lambda})\right|\le q^\dagger(\delta;\varphi)\right)
    =1-\delta.
\end{align}

\section{Proofs for Sections~\ref{sec: prel} and~\ref{sec: limit}}
\subsection{Proof of Theorem~\ref{thm: pa_limit}}
\label{apd: pa_limit}
We first revisit the iteration error. For $k\ge1$,
\begin{align}
    \theta_k-\theta^*&=\theta_{k-1}-\theta^*-\eta_k g(\theta_{k-1},\xi_k)\notag\\
    &=(I_d-\eta_k H)(\theta_{k-1}-\theta^*)
      -\eta_k R_{k-1,1}-\eta_k R_{k-1,2}-\eta_k\delta_k,
\end{align}
where $R_{k-1,1}=\nabla\ell(\theta_{k-1})-H(\theta_{k-1}-\theta^*)$,
$R_{k-1,2}=g(\theta_{k-1},\xi_k)-\nabla\ell(\theta_{k-1})-g(\theta^*,\xi_k)$,
and $\delta_k=g(\theta^*,\xi_k)$. Let $A_k=I_d-\eta_kH$. Iterating the recursion gives
\begin{align}
    \theta_k-\theta^*=\prod_{i=1}^k A_i(\theta_0-\theta^*)
    -\sum_{j=1}^k\prod_{i=j+1}^k A_i\eta_j
    (R_{j-1,1}+R_{j-1,2}+\delta_j).
\end{align}
Then the Polyak--Ruppert averaged estimator satisfies
\begin{align}
    \frac{1}{n}\sum_{k=1}^n(\theta_k-\theta^*)&=\frac{1}{n}\sum_{k=1}^{n}\prod_{i=1}^{k}A_i(\theta_0-\theta^*)-\frac{1}{n}\sum_{k=1}^{n}\sum_{j=1}^{k}\prod_{i=j+1}^{k}A_i\eta_j(R_{j-1,1}+R_{j-1,2}+\delta_j)\notag\\
    &:=\Delta_1(n)-\Delta_2(n)-\Delta_3(n)-\Delta_4(n),
\end{align}
where
\begin{align}
    &\Delta_2(n)=\frac{1}{n}\sum_{k=1}^{n}\sum_{j=1}^{k}\prod_{i=j+1}^{k}A_i\eta_j \delta_j,\notag\\
    &\Delta_3(n)=\frac{1}{n}\sum_{k=1}^{n}\sum_{j=1}^{k}\prod_{i=j+1}^{k}A_i\eta_j R_{j-1,1},\notag\\
    &\Delta_4(n)=\frac{1}{n}\sum_{k=1}^{n}\sum_{j=1}^{k}\prod_{i=j+1}^{k}A_i\eta_j R_{j-1,2}.\notag
\end{align}
Because $\prod_{j=1}^{n}(I_d-\eta_j H)$ decays at a stretched-exponential rate, $\Delta_1(n)$ satisfies
\begin{align}
    \lim_{n\to+\infty}\frac{b_1(n)}{n^{\frac{1}{\alpha}-1}}\|\Delta_1(n)\|=0.
\end{align}
Rearranging the terms in $\Delta_2(n)$ gives
\begin{align}
    \Delta_2(n)&=\frac{1}{n}\sum_{j=1}^{n}\sum_{k=j}^n\prod_{i=j+1}^{k}A_i\eta_j\delta_j\notag\\
    &=\frac{1}{n}\sum_{j=1}^{n}A_j^n\delta_j\notag\\
    &=\frac{1}{n}\sum_{j=1}^{n}H^{-1}\delta_j+\frac{1}{n}\sum_{j=1}^{n}(A_j^n-H^{-1})\delta_j,
\end{align}
where $A_j^n=\eta_j\sum_{k=j}^n\prod_{i=j+1}^{k}A_i$. The process
$S_t^n:=\sum_{j=1}^{t}(A_j^n-H^{-1})\delta_j$ is a martingale for
$t=1,\ldots,n$. For $\beta\in[1,\alpha)$, Lemma~\ref{lem: moment} gives
\begin{align}
    \EB\left\|\frac{b_1(n)}{n^{\frac{1}{\alpha}}}\sum_{j=1}^{n}(A_j^n-H^{-1})\delta_j\right\|^{\beta}&\lesssim\frac{b_1(n)^{\beta}}{n^{\frac{\beta}{\alpha}}}\sum_{j=1}^{n}\|A_j^n-H^{-1}\|^{\beta}\notag\\
    &\lesssim\frac{b_1(n)^{\beta}}{n^{\frac{\beta}{\alpha}}}\sum_{j=1}^{n}\|A_j^n-H^{-1}\|.
\end{align}
If $\max\{\alpha\rho,1\}<\beta<\alpha$, Lemma~\ref{lem: avg_coeff} gives
\begin{align}
    \EB\left\|\frac{b_1(n)}{n^{\frac{1}{\alpha}}}\sum_{j=1}^{n}(A_j^n-H^{-1})\delta_j\right\|^{\beta}\lesssim\frac{b_1(n)^{\beta}}{n^{\frac{\frac{\beta}{\alpha}-\rho}{2}}}\frac{\sum_{j=1}^{n}\|A_j^n-H^{-1}\|}{n^{\frac{\frac{\beta}{\alpha}+\rho}{2}}}\overset{n\to+\infty}{\to}0.
\end{align}
Consequently,
\begin{align}
\label{eq: pa}
    \frac{b_1(n)}{n^{\frac{1}{\alpha}-1}}\Delta_2(n)\tod H^{-1}L_1,
\end{align}
where $L_t$ is a L\'evy process indexed by $t\in[0,1]$ with characteristics $\left(0, \nu(\theta^*,\cdot),-\int_{\|x\|>1}x\nu(\theta^*,dx)\right)$. For the term $\Delta_3(n)$, we have
\begin{align}
    \EB\|\Delta_3(n)\|&=\EB\left\|\frac{1}{n}\sum_{j=1}^{n}A_j^nR_{j-1,1}\right\|\notag\\
    &\lesssim\frac{1}{n}\sum_{j=1}^{n}\EB\|R_{j-1,1}\|\notag\\
    &\lesssim\frac{1}{n}\sum_{j=1}^{n}\EB\|\theta_{j-1}-\theta^*\|^q\notag\\
    &\lesssim o\left(n^{\varepsilon-\rho q\frac{\alpha-1}{\alpha}}\right),
\end{align}
where the last bound follows from Lemma~\ref{lem: final_error}. Thus,
\begin{align}
    \frac{b_1(n)}{n^{\frac{1}{\alpha}-1}}\EB\|\Delta_3(n)\|=o\left(b_1(n) n^{\varepsilon-(\rho q-1)\frac{\alpha-1}{\alpha}}\right).
\end{align}
As $\rho>q^{-1}$, the above term tends to zero as $n\to+\infty$. Thus, $\frac{b_1(n)}{n^{\frac{1}{\alpha}-1}}\Delta_3(n)\overset{p}{\to}0$. We write the term $\Delta_4(n)$ as
\begin{align}
    \Delta_4(n)=\frac{1}{n}\sum_{j=1}^{n}A_j^n R_{j-1,2}.
\end{align}
Let $S_t^n=\sum_{j=1}^{t}A_j^nR_{j-1,2}$. With
$\mathcal F_t=\sigma(\xi_1,\ldots,\xi_t)$, this process satisfies
$\EB[S_t^n\mid\mathcal F_{t-1}]=S_{t-1}^n$. For $\beta\in(1,\alpha)$, Lemma~\ref{lem: moment} gives
\begin{align}
    \EB\|\Delta_4(n)\|^\beta&\lesssim\frac{1}{n^{\beta}}\sum_{j=1}^{n}\EB\|R_{j-1,2}\|^\beta\notag\\
    &\lesssim\frac{1}{n^{\beta}}\sum_{j=1}^{n}\EB\|\theta_{j-1}-\theta^*\|^\beta\notag\\
    &\lesssim\frac{1}{n^{\beta}}\sum_{j=0}^{n-1}o\left(j^{\varepsilon-\rho\beta\frac{\alpha-1}{\alpha}}\right)\notag\\
    &\lesssim o\left(n^{\varepsilon-\rho\beta\frac{\alpha-1}{\alpha}-\beta+1}\right).
\end{align}
Thus,
\begin{align}
    \frac{b_1(n)^{\beta}}{n^{\frac{\beta}{\alpha}-\beta}}\EB\|\Delta_4(n)\|^\beta\lesssim o\left(b_1(n)^{\beta}n^{\varepsilon-\rho\beta\frac{\alpha-1}{\alpha}-\frac{\beta}{\alpha}+1}\right).
\end{align}
As long as $\frac{\alpha}{1+\rho(\alpha-1)}<\beta<\alpha$, we have $\frac{b_1(n)}{n^{\frac{1}{\alpha}-1}}\Delta_4(n)\overset{p}{\to}0$.
Combining the decomposition in equation~\eqref{eq: pa} with the negligible terms gives
\begin{align}
    n^{1-\frac{1}{\alpha}}b_1(n)(\bar\theta_n-\theta^*)
    \tod -H^{-1}L_1.
\end{align}

\subsection{Proof of Theorem~\ref{thm: scale_diff}}
Let $\Gamma$ denote the spectral measure of the $\alpha$-stable random
vector $L_1$. By Definition~\ref{defn: scale}, the scale function of $Z_{\text{Polyak}}$ satisfies ($\|u\|=1$)
\begin{align}
    \sigma_{\alpha}(Z_{\text{Polyak}};u)=\int_{\mathbb{S}^{d-1}} |s^\top H^{-1}u|^\alpha\,\Gamma(ds).
\end{align}
Here and below, the overall minus signs in $Z_{\text{Polyak}}$ and
$Z_{\text{final},1}$ do not affect directional scale because
$\sigma_\alpha(-Z;u)=\sigma_\alpha(Z;u)$.
For any deterministic matrix-valued function $A(t)$, the independent-increment property of $L_t$ gives
\begin{align}
    \sigma_{\alpha}\left(\int_0^\infty A(t)\,dL_t; u\right)
    =
    \int_0^\infty
    \sigma_{\alpha}(L_1; A(t)^\top u)\,dt .
\end{align}
This identity follows first for step functions $A(t)$ and then by the usual approximation argument for deterministic L\'evy integrals. Applying $A(t)=c^{1-1/\alpha}e^{-H_1t}$ yields
\begin{align}
    \sigma_{\alpha}(Z_{\text{final},1};u)=c^{\alpha-1}\int_{\mathbb{S}^{d-1}} \int_0^{+\infty}|s^\top \exp\left(-H_1t\right)u|^\alpha dt\,\Gamma(ds).
\end{align}
The identity $H^{-1}=\int_0^{+\infty}\exp(-Ht)dt$ gives
\begin{align}
    |s^\top H^{-1}u|^\alpha&=\left|\int_0^{+\infty}\exp\left(-\frac{1-\alpha^{-1}}{c}t\right)s^\top\exp\left(-H_1t\right)udt\right|^{\alpha}\\
    &\overset{(a)}{\le}\left|\int_0^{+\infty}\exp\left(-\frac{1}{c}t\right)dt\right|^{\alpha-1}\int_0^{+\infty}\left|s^\top\exp\left(-H_1t\right)u\right|^{\alpha}dt\\
    &=c^{\alpha-1}\int_0^{+\infty}\left|s^\top\exp\left(-H_1t\right)u\right|^{\alpha}dt,
\end{align}
where $(a)$ follows from H\"older's inequality. Integrating both sides with respect to $\Gamma$ gives
\begin{align}
    \sigma_{\alpha}(Z_{\text{Polyak}};u)\le \sigma_{\alpha}(Z_{\text{final},1};u).
\end{align}
This proves the desired directional scale dominance.

\subsection{Proof of Theorem~\ref{thm: joint}}
\begin{lem}
\label{lem: variance}
    The empirical second-moment normalizer $\Sigma_n$ satisfies the following decomposition:
    \begin{align}
        n^{1-\frac{2}{\alpha}}b_1(n)^2\Sigma_n = \frac{b_1(n)^2}{n^{\frac{2}{\alpha}}}\sum_{i=1}^ng(\theta^*,\xi_i)g(\theta^*,\xi_i)^\top +o_p(1).
    \end{align}
\end{lem}
\begin{proof}
    We decompose $g(\theta_{i-1},\xi_i)$ as
    \begin{align}
        g(\theta_{i-1},\xi_i)=g(\theta^*,\xi_i)+
        \bigl(g(\theta_{i-1},\xi_i)-g(\theta^*,\xi_i)\bigr):=X_i+Y_i.
    \end{align}
    Then
    \begin{align}
        \Sigma_n &= \frac{1}{n}\sum_{i=1}^nX_i X_i^\top+\frac{1}{n}\sum_{i=1}^nX_i Y_i^\top + \frac{1}{n}\sum_{i=1}^nY_i X_i^\top+\frac{1}{n}\sum_{i=1}^nY_i Y_i^\top\\
        &:= \frac{1}{n}\sum_{i=1}^nX_i X_i^\top + \Delta_n.
    \end{align}
    For any $u\in\mathbb{S}^{d-1}$, multiplying both sides by $n^{1-\frac{2}{\alpha}}b_1(n)^2$ gives
    \begin{align}
        \frac{b_1(n)^2}{n^{\frac{2}{\alpha}-1}}u^\top\Delta_n u=\frac{2b_1(n)^2}{n^{\frac{2}{\alpha}}}\sum_{i=1}^n(u^\top X_i)(u^\top Y_i)+\frac{b_1(n)^2}{n^{\frac{2}{\alpha}}}\sum_{i=1}^n(u^\top Y_i)^2.
    \end{align}
    Assumption~\ref{asmp: lip_error} gives $\|Y_i\|\le C_{\rm Lip}\|\theta_{i-1}-\theta^*\|$. Thus, for $1<p<\alpha$ we have
    \begin{align}
        \frac{b_1(n)}{n^{\frac{1}{\alpha}}}\EB\sqrt{\sum_{i=1}^n(u^\top Y_i)^2}&\le\frac{b_1(n)}{n^{\frac{1}{\alpha}}}\EB\left(\sum_{i=1}^n \|Y_i\|^p\right)^{\frac{1}{p}}\\
        &\le\frac{b_1(n)}{n^{\frac{1}{\alpha}}}\left(\sum_{i=1}^n \EB\|Y_i\|^p\right)^{\frac{1}{p}}\\
        &\lesssim\frac{b_1(n)}{n^{\frac{1}{\alpha}}}\left(1+\sum_{i=2}^n \eta_{i-1}^{p-1}\right)^{\frac{1}{p}}\\
        &\le b_1(n)n^{\frac{1-(p-1)\rho}{p}-\frac{1}{\alpha}}.
    \end{align}
    Thus, choosing $p\in\left(\frac{1+\rho}{\alpha^{-1}+\rho},\alpha\right)$ gives
    \begin{align}
        \frac{b_1(n)^2}{n^{\frac{2}{\alpha}}}\sum_{i=1}^n(u^\top Y_i)^2\overset{p}{\to}0.
    \end{align}
    Similarly,
    \begin{align}
        \frac{2b_1(n)^2}{n^{\frac{2}{\alpha}}}\sum_{i=1}^n\EB|u^\top X_i||u^\top Y_i|&\lesssim\frac{b_1(n)^2}{n^{\frac{2}{\alpha}}}\sum_{i=1}^n\EB\|X_i\|\|\theta_{i-1}-\theta^*\|\\
        &=\frac{b_1(n)^2}{n^{\frac{2}{\alpha}}}\sum_{i=1}^n\EB\|X_i\|\EB\|\theta_{i-1}-\theta^*\|\\
        &\lesssim\frac{b_1(n)^2}{n^{\frac{2}{\alpha}}}\left(1+\sum_{i=2}^n\eta_{i-1}^{1-\frac{1}{p}}\right)\\
        &\lesssim b_1(n)^2 n^{1-(1-\frac{1}{p})\rho-\frac{2}{\alpha}}\to0.
    \end{align}
    Consequently,
    \begin{align}
        \frac{2b_1(n)^2}{n^{\frac{2}{\alpha}}}\sum_{i=1}^n(u^\top X_i)(u^\top Y_i)\overset{p}{\to}0.
    \end{align}
    Thus, for every $u\in\mathbb{S}^{d-1}$, $\frac{b_1(n)^2}{n^{\frac{2}{\alpha}-1}}u^\top\Delta_n u=o_p(1)$. Applying this to a fixed basis and using polarization shows that every entry is $o_p(1)$ after scaling; equivalence of matrix norms in fixed dimension then gives $\frac{b_1(n)^2}{n^{\frac{2}{\alpha}-1}}\|\Delta_n\|=o_p(1)$.
\end{proof}

By Theorem~\ref{thm: pa_limit}, the dominant random term for the Polyak--Ruppert averaged estimator is
\begin{align}
    n^{1-\frac{1}{\alpha}}b_1(n)\left(\frac{\sum_{k=1}^n\theta_k}{n}-\theta^*\right)=-\frac{b_1(n)}{n^{\frac{1}{\alpha}}}\sum_{k=1}^{n}H^{-1}g(\theta^*,\xi_k)+o_p(1).
\end{align}
By Lemma~\ref{lem: variance}, the dominant random term for $\Sigma_n$ is
\begin{align}
    n^{1-\frac{2}{\alpha}}b_1(n)^2\Sigma_n = \frac{b_1(n)^2}{n^{\frac{2}{\alpha}}}\sum_{k=1}^{n}g(\theta^*,\xi_k)g(\theta^*,\xi_k)^\top+o_p(1).
\end{align}
Since $\EB g(\theta^*,\xi)=\nabla\ell(\theta^*)=0$, the centering condition in Lemma~\ref{lem: joint_vector} is satisfied. The joint convergence then follows from that lemma.
Moreover, $W$ is a nonzero positive-semidefinite stable random matrix and
$\PB(\Tr(W)>0)=1$; hence the self-normalizing map is almost surely continuous
at the joint limit.

\subsection{Proof of Corollary~\ref{cor: linear}}

\begin{proof}
    Set $X_n=n^{1-\frac{1}{\alpha}}b_1(n)(\bar\theta_n-\theta^*)$ and $A_n=n^{1-\frac{2}{\alpha}}b_1(n)^2\Sigma_n$. Theorem~\ref{thm: joint} gives $(X_n,A_n)\tod(-H^{-1}L_1,W)$. Since matrix inversion is continuous on the set of nonsingular matrices and $W$ is invertible almost surely, $A_n$ is invertible with probability tending to one and $(X_n,A_n^{-1})\tod(-H^{-1}L_1,W^{-1})$. The identity $T_n^\dagger=X_n^\top A_n^{-1}X_n$ and the continuous mapping theorem yield the result.
\end{proof}

\subsection{Proof of Theorem~\ref{thm: sub_consistent}}

\begin{proof}
We first introduce the infeasible version centered at the true parameter $\theta^*$:
\[
    \widetilde T_{n,\varphi}^{(b)}
    :=
    \frac{
        \sqrt{t_n}\,
        \left|\varphi\left(\bar\theta_{t_n}^{(b)}-\theta^*\right)\right|
    }{
        \sqrt{\Tr\left(\Sigma_{t_n}^{(b)}\right)}
    },
\]
and define
\[
    \widetilde F_n(x;\varphi)
    :=
    \frac{1}{B_n}
    \sum_{b=1}^{B_n}
    \mathbbm{1}\left\{
        \widetilde T_{n,\varphi}^{(b)}
        \le x
    \right\}.
\]
We first show that
\[
    \widetilde F_n(x;\varphi)
    \overset{p}{\to}
    \PB\left(Z^\star(\varphi)\le x\right)
\]
for every continuity point \(x\) of the distribution function of \(Z^\star(\varphi)\). Indeed, by Theorem~\ref{thm: joint} applied to
each length-\(t_n\) sub-procedure,
\[
    \widetilde T_{n,\varphi}^{(b)}
    \tod
    Z^\star(\varphi),
\]
as \(t_n\to\infty\). Therefore, for every continuity point \(x\),
\[
    \PB\left(
            \widetilde T_{n,\varphi}^{(b)}\le x
    \right)
    \to
    \PB\left(Z^\star(\varphi)\le x\right).
\]
Moreover, since the sub-procedures are independent,
\[
    \Var\left(
        \widetilde F_n(x;\varphi)
    \right)
    =
    \frac{1}{B_n}
    \Var\left(
        \mathbbm{1}\left\{
            \widetilde T_{n,\varphi}^{(1)}\le x
        \right\}
    \right)
    \le
    \frac{1}{B_n}
    \to 0.
\]
Hence
\[
    \widetilde F_n(x;\varphi)
    -
    \EB \widetilde F_n(x;\varphi)
    \overset{p}{\to} 0,
\]
which proves
\[
    \widetilde F_n(x;\varphi)
    \overset{p}{\to}
    \PB\left(Z^\star(\varphi)\le x\right).
\]
It remains to show that replacing the infeasible centering \(\theta^*\) by
\(\bar\theta_n\) is asymptotically negligible. Define
\[
    D_{n,\varphi}^{(b)}
    :=
    \frac{
        \sqrt{t_n}\,
        \left|\varphi\left(\bar\theta_n-\theta^*\right)\right|
    }{
        \sqrt{\Tr\left(\Sigma_{t_n}^{(b)}\right)}
    } .
\]
By Theorem~\ref{thm: joint},
\[
    n^{1-\frac{1}{\alpha}}b_1(n)(\bar\theta_n-\theta^*)
    =
    O_p(1),
\]
and
\[
    t_n^{1-\frac{2}{\alpha}}b_1(t_n)^2
    \Sigma_{t_n}^{(b)}
    \tod W .
\]
Since \(\PB(\Tr(W)>0)=1\) and \(\varphi\) is homogeneous and
continuous, we obtain
\[
    D_{n,\varphi}^{(b)}
    =
    O_p\left(
        \frac{
            t_n^{1-\frac{1}{\alpha}}b_1(t_n)
        }{
            n^{1-\frac{1}{\alpha}}b_1(n)
        }
    \right)
    =
    o_p(1),
\]
where the last step follows from \(t_n/n\to 0\) and the regular variation of
\(n^{1-\frac{1}{\alpha}}b_1(n)\). Consequently, for every \(\varepsilon>0\),
\[
    \PB\left(D_{n,\varphi}^{(b)}>\varepsilon\right)\to 0.
\]
By the evenness and subadditivity of $|\varphi(\cdot)|$,
\[
    \left|
        T_{n,\varphi}^{\star,(b)}
        -
        \widetilde T_{n,\varphi}^{(b)}
    \right|
    \le
    D_{n,\varphi}^{(b)}.
\]
Thus, for every \(\varepsilon>0\),
\[
    \mathbbm{1}\left\{
        \widetilde T_{n,\varphi}^{(b)}\le x-\varepsilon
    \right\}
    -
    \mathbbm{1}\left\{
        D_{n,\varphi}^{(b)}>\varepsilon
    \right\}
    \le
    \mathbbm{1}\left\{
        T_{n,\varphi}^{\star,(b)}\le x
    \right\}
\]
and
\[
    \mathbbm{1}\left\{
        T_{n,\varphi}^{\star,(b)}\le x
    \right\}
    \le
    \mathbbm{1}\left\{
        \widetilde T_{n,\varphi}^{(b)}\le x+\varepsilon
    \right\}
    +
    \mathbbm{1}\left\{
        D_{n,\varphi}^{(b)}>\varepsilon
    \right\}.
\]
Averaging over \(b=1,\ldots,B_n\), we obtain
\[
    \widetilde F_n(x-\varepsilon;\varphi)-A_n(\varepsilon)
    \le
    \widehat F_n(x;\varphi)
    \le
    \widetilde F_n(x+\varepsilon;\varphi)+A_n(\varepsilon),
\]
where
\[
    A_n(\varepsilon)
    :=
    \frac{1}{B_n}
    \sum_{b=1}^{B_n}
    \mathbbm{1}\left\{
        D_{n,\varphi}^{(b)}>\varepsilon
    \right\}.
\]
Since
\[
    \EB A_n(\varepsilon)
    \le
    \sup_{1\le b\le B_n}\PB\left(D_{n,\varphi}^{(b)}>\varepsilon\right)
    \to 0,
\]
Markov's inequality implies
\[
    A_n(\varepsilon)\overset{p}{\to} 0.
\]
Therefore,
\[
    \widetilde F_n(x-\varepsilon;\varphi)+o_p(1)
    \le
    \widehat F_n(x;\varphi)
    \le
    \widetilde F_n(x+\varepsilon;\varphi)+o_p(1).
\]
Choose a sequence $\varepsilon_m\downarrow0$ such that both
$x-\varepsilon_m$ and $x+\varepsilon_m$ are continuity points of the limiting
distribution. Letting $n\to\infty$ and then $m\to\infty$, and using continuity
at $x$, we conclude that
\[
    \widehat F_n(x;\varphi)
    \overset{p}{\to}
    \PB\left(Z^\star(\varphi)\le x\right).
\]
Finally, let
\[
    F_\varphi(x):=\PB\left(Z^\star(\varphi)\le x\right),
    \qquad
    q(\delta;\varphi)
    :=
    \inf\left\{
        x:F_\varphi(x)\ge 1-\delta
    \right\},
\]
and
\[
    \widehat q(\delta;\varphi)
    :=
    \inf\left\{
        x:\widehat F_n(x;\varphi)\ge 1-\delta
    \right\}.
\]
Suppose that $F_\varphi$ is continuous and strictly increasing at
$q=q(\delta;\varphi)$. Then, for every $\varepsilon>0$,
\begin{align}
    F_\varphi(q-\varepsilon)<1-\delta<F_\varphi(q+\varepsilon).
    \label{eq: quantile_identification}
\end{align}
For each $\varepsilon$, choose continuity points
$x_-\in(q-\varepsilon,q)$ and $x_+\in(q,q+\varepsilon)$. The pointwise
convergence of $\widehat F_n$ at $x_-$ and $x_+$ implies
\[
    \widehat q(\delta;\varphi)\overset{p}{\to} q(\delta;\varphi).
\]
Moreover, by Theorem~\ref{thm: joint},
\[
    |T_n^\star(\theta^*;\varphi)|
    \tod
    Z^\star(\varphi).
\]
Hence, by Slutsky's theorem,
\[
\begin{aligned}
    \PB\left(\theta^*\in \widehat{\RM}_n(\delta;\varphi)\right)
    &=
    \PB\left(
        |T_n^\star(\theta^*;\varphi)|
        \le
        \widehat q(\delta;\varphi)
    \right)  \\
    &\to
    \PB\left(
        Z^\star(\varphi)\le q(\delta;\varphi)
    \right)
    =
    1-\delta .
\end{aligned}
\]
This proves the desired asymptotic coverage.

% If $F_\varphi$ is discontinuous at $q$ but is strictly increasing there,
% the same bracketing argument
% continues to give $\widehat q(\delta;\varphi)\overset{p}{\to}q$. However,
% weak convergence alone does not control how the estimated critical value
% approaches the atom. For fixed $a>0$ such that $q+a$ is a continuity point,
% Slutsky's theorem instead gives
% \begin{align}
%     \PB\left(
%         |T_n^\star(\theta^*;\varphi)|
%         \le \widehat q(\delta;\varphi)+a
%     \right)
%     \longrightarrow F_\varphi(q+a)
%     \ge F_\varphi(q)
%     \ge1-\delta.
% \end{align}
% This proves the conservative buffered-critical-value assertion following
% Theorem~\ref{thm: sub_consistent}.
\end{proof}

\section{Auxiliary lemmas}
\begin{lem}[Lemma 7 in \cite{wang2021convergence}]
    \label{lem: moment}
    Suppose $p\in[1,2]$ and $\{S_t\}$ is a $d$-dimensional martingale with finite $p$th moments for every $t$, and $S_0=0$. Then
    \begin{align*}
        \EB\|S_t\|^{p}\le2^{2-p}d^{1-\frac{p}{2}}\sum_{i=1}^{t}\EB\|S_i-S_{i-1}\|^{p}.
    \end{align*}
\end{lem}

\begin{lem}
\label{lem: avg_coeff}
    Suppose $H$ is a symmetric positive definite matrix and $\eta_n=c\cdot n^{-\rho}$, where $\rho\in(0,1)$. Define $A_k=I_d-\eta_k H$ and $A_j^n=\eta_j\sum_{k=j}^n\prod_{i=j+1}^{k} A_i$. Then
    \begin{enumerate}[(a)]
        \item \text{\rm(Lemma 1 in \cite{polyak1992acceleration})} $\sup_{1\le j\le n}\|A_j^n\|\le C_A$.
        \item \text{\rm(Lemma 17 in \cite{wang2021convergence})} for any $\rho<\kappa\le1$, 
        \begin{align}
        \lim_{n\to+\infty}\frac{\sum_{j=1}^{n}\|A_j^n-H^{-1}\|}{n^\kappa}=0.
        \end{align}
    \end{enumerate}
\end{lem}

\begin{lem}[Corollary 4 in \cite{wang2021convergence}]
\label{lem: final_error}
    For any $\varepsilon>0$ and $q\in(1,\alpha)$,
    \begin{align}
        \EB\|\theta_n-\theta^*\|^q=o\left(n^{\varepsilon-\rho q\frac{\alpha-1}{\alpha}}\right).
    \end{align}
\end{lem}

\begin{lem}[Lemma 5 in \cite{meerschaert1999sample}]
    \label{lem: joint_vector}
    For an i.i.d. sequence of multivariate $\alpha$-regularly varying random vectors $\{X_i\}_{i\ge1}$ with $\EB X_1=0$, the joint weak convergence holds:
    \begin{align}
        \left(\frac{b_1(n)}{n^{\frac{1}{\alpha}}}\sum_{i=1}^n X_i,\frac{b_1(n)^2}{n^{\frac{2}{\alpha}}}\sum_{i=1}^n X_iX_i^\top \right)\tod (Z, W).
    \end{align}
\end{lem}

\section{Additional experimental results}
\subsection{Simulation: robustness to varying tail indices}
\label{apd: robust}
The theoretical analysis in this paper assumes that the stochastic gradient noise follows a heavy-tailed distribution with a fixed tail index $\alpha\in(1,2)$. This assumption enables us to characterize the asymptotic behavior of the Polyak--Ruppert averaged SGD estimator via stable limit theorems. However, in practice, the heaviness of the stochastic gradient noise may vary over time. For instance, its distribution may change across iterations because of different sampled mini-batches or evolving model parameters. Consequently, the tail index need not remain constant throughout the optimization process.

To evaluate the robustness of our inference procedure under such heterogeneous noise regimes, we consider a setting in which the tail index varies randomly across iterations. Specifically, at iteration $k$, the stochastic gradient noise follows a Pareto distribution with tail index $\alpha_k$, where $\alpha_k$ is independently drawn from a uniform distribution. We consider three ranges for $\alpha_k$: $[1.5,1.8)$, $[1.8,2.2)$, and $[2.2,4)$, corresponding to a heavy-tailed regime, a mixed heavy- and light-tailed regime, and a light-tailed regime, respectively. This setting departs from the theoretical framework developed in the earlier sections, where the tail index is fixed. We apply the varying-index setting to both linear and logistic regression. For logistic regression, we consider homogeneous and heterogeneous settings. In the homogeneous setting, a single $\alpha_k$ is sampled at each iteration and applied to all coordinates, whereas in the heterogeneous setting, different $\alpha_k$ values are sampled across coordinates.

\begin{table}[t]
\caption{Varying tail indices: Average coverage rates and confidence interval lengths at the nominal 95\% coverage level. Standard errors are reported in parentheses.}
\resizebox{\textwidth}{!}{\begin{tabular}{lcccccc}
\toprule
 &  \multicolumn{3}{c}{Coverage Rate (\%)}  &  \multicolumn{3}{c}{Average Length ($\times 10^{-2}$)} \\
  & $n^{0.6}$&$n^{0.7}$&$n^{0.8}$ & $n^{0.6}$&$n^{0.7}$&$n^{0.8}$  \\
\midrule
\multicolumn{7}{l}{Linear Regression (Identity $\Sigma$)}\\
$\alpha\in[1.5,1.8)$&96.2(8.6E-3)&94.6(1.0E-2)&93.0(1.1E-2)&4.828(3.6E-3)&5.161(4.9E-3)&5.657(7.3E-3)\\
$\alpha\in[1.8,2.2)$&93.8(1.1E-2)&92.4(1.2E-2)&87.6(1.5E-2)&1.136(2.3E-4)&1.119(2.6E-4)&1.065(2.9E-4)\\
$\alpha\in[2.2,4)$&94.6(1.0E-2)&94.4(1.0E-2)&86.4(1.5E-2)&2.644(4.5E-6)&2.529(8.1E-6)&2.327(1.3E-5)\\

 \addlinespace
\multicolumn{7}{l}{Linear Regression (Toeplitz $\Sigma$)}\\
$\alpha\in[1.5,1.8)$&96.6(8.1E-3)&94.6(1.0E-2)&91.8(1.2E-2)&3.940(2.4E-3)&3.892(1.8E-3)&4.083(3.2E-3)\\
$\alpha\in[1.8,2.2)$&94.2(1.0E-2)&90.6(1.3E-2)&87.8(1.5E-2)&1.248(2.5E-4)&1.217(2.9E-4)&1.156(3.4E-4)\\
$\alpha\in[2.2,4)$&94.6(1.0E-2)&92.6(1.2E-2)&87.0(1.5E-2)&2.877(4.8E-6)&2.722(8.9E-6)&2.494(1.4E-5)\\

 \addlinespace
 \multicolumn{7}{l}{Logistic Regression (Homogeneous $x_i$)}\\
$\alpha\in[1.5,1.8)$&94.0(1.1E-2)&94.6(1.0E-2)&94.2(1.0E-2)&19.03(2.0E-2)&31.25(4.0E-2)&65.88(8.6E-2)\\
$\alpha\in[1.8,2.2)$&99.6(2.8E-3)&99.2(4.0E-3)&97.8(6.6E-3)&5.012(1.1E-3)&4.335(2.5E-3)&8.688(1.2E-2)\\
$\alpha\in[2.2,4)$&99.8(2.0E-3)&99.2(4.0E-3)&94.2(1.0E-2)&2.016(5.8E-5)&1.446(9.9E-5)&1.211(1.0E-3)\\

\addlinespace
 \multicolumn{7}{l}{Logistic Regression (Heterogeneous $x_i$)}\\
$\alpha\in[1.5,1.8)$&93.6(1.1E-2)&95.2(9.6E-3)&94.6(1.0E-2)&18.25(1.3E-2)&31.65(2.7E-2)&55.61(6.1E-2)\\
$\alpha\in[1.8,2.2)$&99.2(4.0E-3)&98.6(5.3E-3)&97.0(7.6E-3)&5.000(8.1E-4)&4.312(1.7E-3)&8.271(1.5E-2)\\
$\alpha\in[2.2,4)$&100.0(0E-0)&99.8(2.0E-3)&94.4(1.0E-2)&2.025(6.5E-5)&1.465(1.0E-4)&1.498(2.8E-3)\\
\bottomrule
\end{tabular}}
\label{table: varying}
\end{table}

Nevertheless, it is plausible that the asymptotic behavior of the averaged SGD estimator remains governed by the heaviest tail in the system. A rigorous theoretical analysis of such varying-tail regimes is beyond the scope of the present work and is left for future research. Empirically, however, the proposed inference procedure continues to achieve reliable coverage probabilities in both linear and logistic regression models, as shown in Table~\ref{table: varying}. We summarize the results below.
\begin{enumerate}[(a)]
    \item[(a)] Across all models considered, including linear regression with identity and Toeplitz covariance structures and logistic regression with homogeneous and heterogeneous covariates, the empirical coverage rates are generally reasonably close to the nominal 95\% level. This suggests that the proposed inference procedure remains robust even when the tail index varies randomly across iterations.
    
    \item[(b)] In the linear regression experiments, the coverage rates are fairly stable when the tail indices are relatively heavy, namely $\alpha\in[1.5,1.8)$. As the noise becomes lighter-tailed, however, the coverage rates tend to deteriorate for larger subsample sizes, especially when the subsample size is chosen as $n^{0.8}$.

    \item[(c)] In the logistic regression experiments, the procedure often yields conservative coverage when the tail indices are in the lighter-tailed regimes $\alpha\in[1.8,2.2)$ and $\alpha\in[2.2,4)$, with empirical coverage rates frequently exceeding 99\%. This phenomenon is more pronounced for smaller subsample sizes, indicating that the method may overestimate uncertainty in these settings.
    
    \item[(d)] The effect of the subsample size is model dependent. For linear regression, the coverage rates generally decrease as the subsample size increases in the range $r\in(0.6,0.8)$, especially under lighter-tailed noise. In contrast, for logistic regression with the heaviest-tailed regime $\alpha\in[1.5,1.8)$, the coverage rates remain relatively stable across different subsample sizes. Overall, these results again reflect a bias--variance trade-off in the choice of subsample size.
    
    \item[(e)] The average confidence interval length decreases substantially as the stochastic gradient noise becomes lighter-tailed. This trend is consistent across both linear and logistic regression models, and agrees with the theoretical intuition that heavier tails lead to slower concentration and hence wider confidence intervals.
\end{enumerate}

\FloatBarrier
\subsection{CIFAR-10: experimental setup}
\label{sec: setup}
\paragraph*{\textbf{Pretraining.}} We use a CIFAR-10-adapted AlexNet architecture following \citet{simsekli2019tail}.
The network is trained on the CIFAR-10 training set using stochastic gradient descent
with a mini-batch size of $128$ and a learning rate of $0.1$ for $100$ epochs. After pretraining,
we freeze the network and use the output of the penultimate layer as the feature vector.
In our implementation, this gives a $448$-dimensional representation for each image.
We then normalize the extracted features coordinate-wise using the empirical mean and
standard deviation computed from the training data. In Section~\ref{sec: fullalex}, we further report an additional experiment based on the
full AlexNet architecture of \citet{krizhevsky2012imagenet}, which produces a
$4096$-dimensional penultimate-layer feature representation.

\paragraph*{\textbf{Subsampling.}}
Using the normalized AlexNet features, we construct a binary classification task by
selecting one CIFAR-10 class as the positive class and grouping the remaining classes as
the negative class. We then fit a regularized logistic regression model on top of the
extracted features. No intercept term is included, so the parameter dimension is exactly
$448$. Specifically, we consider
\[
    \min_{\theta\in\mathbb{R}^{448}}
    \frac{1}{N}\sum_{i=1}^N
    \left\{
        \log\left(1+\exp(z_i^\top\theta)\right)
        - y_i z_i^\top\theta
    \right\}
    + \frac{\lambda}{2}\|\theta\|_2^2 ,
\]
where $z_i$ denotes the normalized feature vector and $y_i\in\{0,1\}$ denotes the binary
label. We set the $\ell_2$ regularization coefficient to $\lambda=0.1$.

Since the true optimizer $\theta^*$ is unknown in this real-data experiment, we compute
a high-accuracy reference solution using L-BFGS and use it as a proxy for $\theta^*$.
The L-BFGS optimization tolerance is set to $10^{-10}$. For the SGD trajectories used in
the inference procedure, we use mini-batches of size $64$ and the learning-rate schedule $\eta_t = 0.5\, t^{-0.6}$.
Each trajectory is run for $10^6$ iterations, and Polyak--Ruppert averaging is applied to
the iterates. 
\begin{figure}[t]
\caption{Tail diagnostics for the $\ell_2$-norm of stochastic gradients in the full AlexNet CIFAR-10 experiment. The stochastic gradients are evaluated near the reference parameter using mini-batches of size $64$. Left: empirical histogram of the gradient norms. Right: Hill estimates of the tail index $\alpha$.}
    \centering
    \begin{minipage}{0.48\textwidth}
        \centering
        \includegraphics[width=\textwidth]{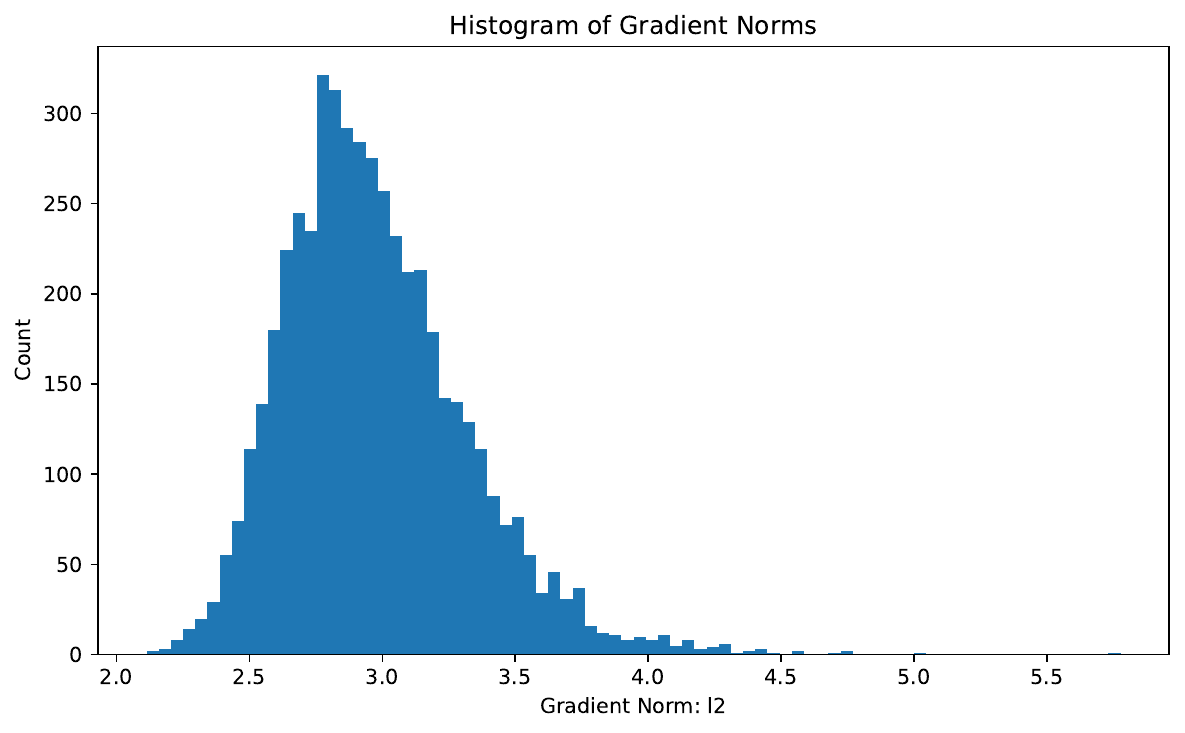}
    \end{minipage}
    \hfill
    \begin{minipage}{0.48\textwidth}
        \centering
        \includegraphics[width=\textwidth]{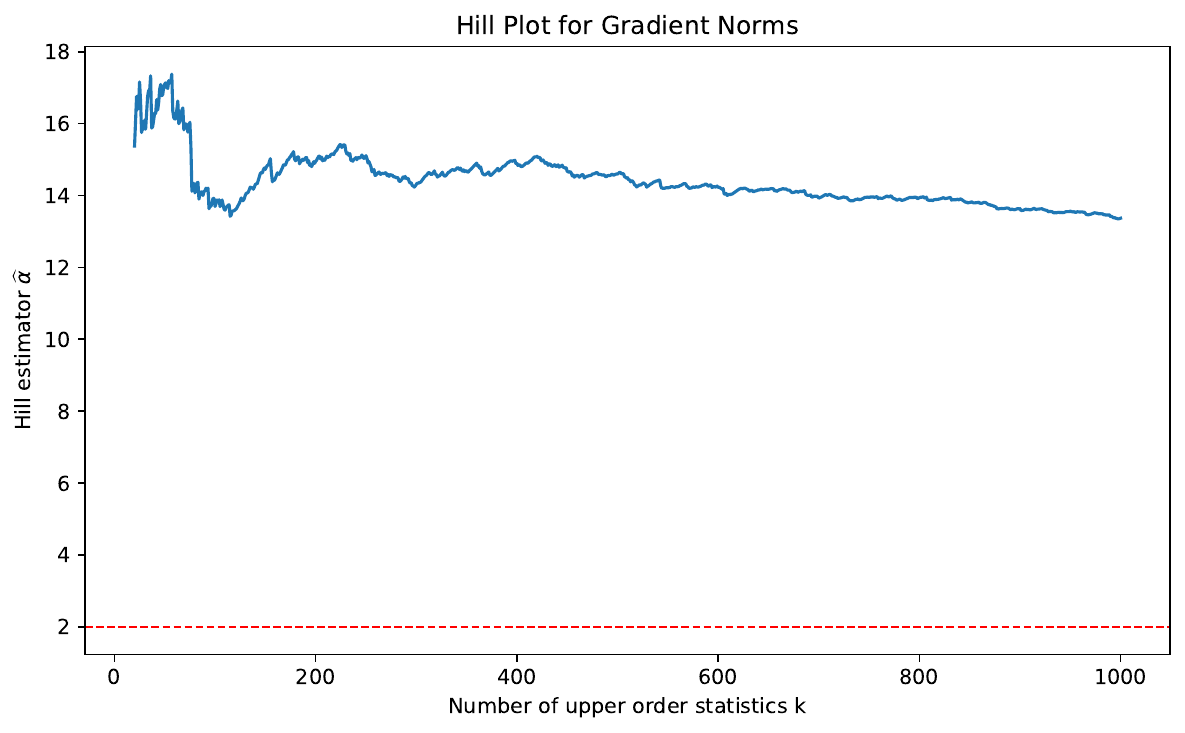}
    \end{minipage}
    \label{fig: full_tail}
\end{figure}
\paragraph*{\textbf{Random scaling.}}
As a benchmark, we compare the proposed method with the random scaling method of
\citet{lee2022fast}. The same regularized logistic regression model is used, again without
an intercept term. The training hyperparameters are the same as those used for the subsampling approach; only the critical value differs. In \citet{lee2022fast}, for an SGD trajectory $\{\theta_t\}_{t=1}^n$, the random-scaling variance estimator is
\[
    \widehat V_n^{\mathrm{RS}}
    =
    \frac{1}{n^2}\sum_{s=1}^n
    s^2(\bar{\theta}_s-\bar\theta_n)(\bar{\theta}_s-\bar\theta_n)^\top .
\]
For coordinate-wise nominal $95\%$ confidence intervals, we use
\[
    \bar\theta_{n,j}
    \pm
    q_{0.975}^{\mathrm{RS}}
    \sqrt{\frac{\widehat V^{\mathrm{RS}}_{n,jj}}{n}},
\]
where $q_{0.975}^{\mathrm{RS}}=6.747$ is the tabulated random-scaling critical value used in \cite{lee2022fast}.

\subsection{CIFAR-10: full AlexNet}
\label{sec: fullalex}
\begin{figure}[t]
    \centering
    \caption{Empirical reference coverage, coverage MSE, and fraction within the 95\% band in the CIFAR-10 experiment with full AlexNet.}
    \includegraphics[width=0.9\linewidth]{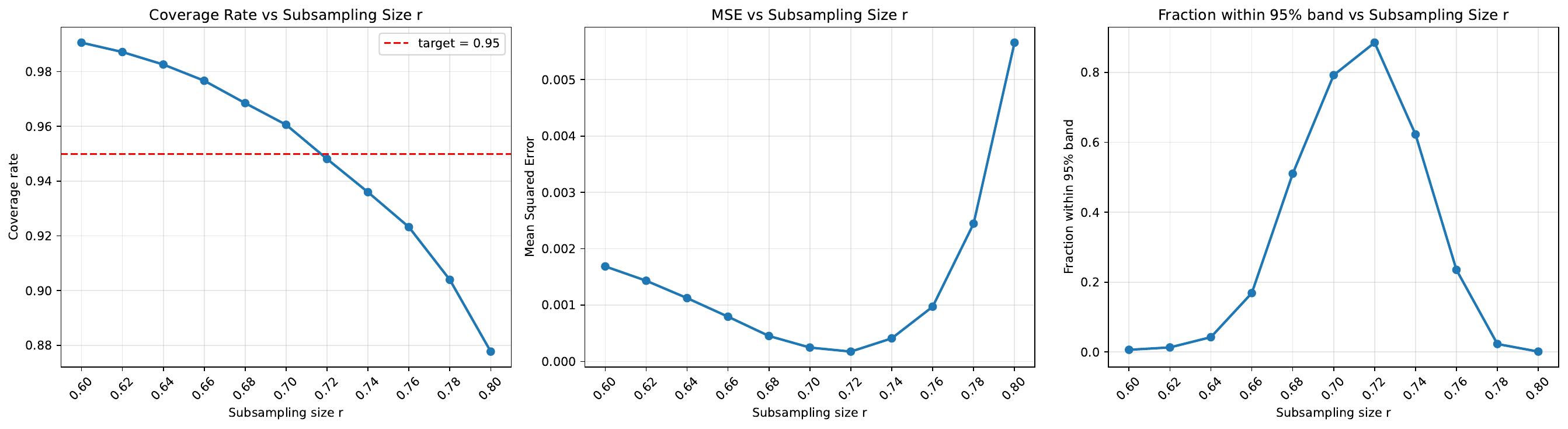}
    \label{fig: cifar10_sub_coverage}
\end{figure}

We conduct an additional CIFAR-10 experiment using the original AlexNet
architecture of \citet{krizhevsky2012imagenet}. This experiment is intended as a
robustness check for the small AlexNet experiment reported in Section~\ref{sec: real}.
The experimental protocol is the same as that described in the preceding subsection,
except that the feature extractor is replaced by the original AlexNet model. In Figure~\ref{fig: full_tail}, we observe that, with the larger feature dimension, the gradient-norm diagnostics appear more consistent with a lighter-tailed, possibly finite-variance, regime. Nevertheless, our methodology remains directly applicable. The corresponding inference results are reported below.

Figure~\ref{fig: cifar10_sub_coverage} reports the performance of the proposed method for
different subsampling levels $r\in\{0.60,0.62,\ldots,0.78,0.80\}$.
The average coverage rate decreases as $r$ increases: smaller values of $r$ tend to produce
conservative intervals, whereas larger values of $r$ lead to under-coverage. The coverage
MSE exhibits a U-shaped pattern, decreasing as $r$ increases from $0.60$ to around
$0.72$ and then increasing for larger values of $r$. This pattern is consistent with the
bias--variance trade-off in subsampling: smaller blocks may yield conservative inference,
whereas larger blocks reduce the effective number of blocks available for estimating the
critical value. Based on the coverage MSE, we select $r=0.72$ for the comparison below.

\begin{table}[t]
\caption{Coverage rate, mean squared error, and fraction within the 95\% band for nominal 95\% confidence intervals on the CIFAR-10 dataset with full AlexNet.}
\centering
\resizebox{\textwidth}{!}{\begin{tabular}{lccc}
\toprule
 & Average Coverage (\%) & Mean Squared Error ($\times 10^{-4}$) & Fraction within $95\%$ Band (\%) \\
\midrule
\texttt{Random Scaling \citep{lee2022fast}} & 93.07& 5.90& 46.97\\
\texttt{Algorithm~\ref{alg: general} ($r=0.72$)}&94.81&1.76&88.53\\
\bottomrule
\end{tabular}}
\label{table: cifar10_full_binary}
\end{table}

Table~\ref{table: cifar10_full_binary} compares the proposed method with random scaling.
Random scaling yields an average coverage rate of $93.07\%$ for nominal $95\%$
confidence intervals, whereas Algorithm~\ref{alg: general} with $r=0.72$ yields
$94.81\%$. The proposed method also reduces the coverage MSE from
$5.90\times 10^{-4}$ to $1.76\times 10^{-4}$ and increases the fraction within the
nominal $95\%$ Monte Carlo band from $46.97\%$ to $88.53\%$.
\FloatBarrier

\end{document}